\PassOptionsToPackage{inline}{enumitem}

\documentclass[10pt]{article} 

\usepackage[accepted]{rlj} 

\usepackage{amssymb}            
\usepackage{mathtools}          
\usepackage{mathrsfs}           
\usepackage{graphicx}           
\usepackage{subcaption}         
\usepackage[space]{grffile}     
\usepackage{url}                
\usepackage{lipsum}             

\usepackage{amsmath}
\usepackage{amsthm}
\usepackage{booktabs} 
\usepackage{algorithm}
\usepackage{algorithmicx}
\usepackage{algpseudocode}
\usepackage[capitalize]{cleveref}
\usepackage[normalem]{ulem}
\usepackage{hyperref}
\usepackage[binary-units]{siunitx}

\crefname{section}{Sec.}{Secs.}
\Crefname{section}{Sec.}{Secs.}
\crefname{algocf}{alg.}{algs.}
\Crefname{algocf}{Algorithm}{Algorithms}

\theoremstyle{plain}
\newtheorem{theorem}{Theorem}[section]
\newtheorem{proposition}[theorem]{Proposition}
\newtheorem{lemma}[theorem]{Lemma}

\theoremstyle{definition}

\newtheorem{assumption}[theorem]{Assumption}

\DeclareMathOperator*{\argmin}{arg\,min}

\newcommand{\rvar}[1]{\mathrm{#1}}        
\newcommand{\rvec}[1]{\mathbf{#1}}        
\renewcommand{\vec}[1]{{\boldsymbol #1}}  
\newcommand{\mat}[1]{{\boldsymbol #1}}    

\newcommand{\PP}[2][]{\mathbb{P}_{#1}\left[#2\right]} 

\newcommand{\A}{\mathcal{A}}
\newcommand{\B}{\mathcal{B}}

\newcommand{\F}{\mathcal{F}}

\renewcommand{\L}{\mathcal{L}}
\newcommand{\M}{\mathcal{M}}

\renewcommand{\O}{\mathcal{O}} 

\renewcommand{\S}{\mathcal{S}}

\newcommand{\X}{\mathcal{X}}
\newcommand{\Y}{\mathcal{Y}}

\title{Risk-Aware General-Utility\\Markov Decision Processes}

\setrunningtitle{Risk-Aware General-Utility Markov Decision Processes}

\author{Pedro P. Santos\textsuperscript{1,2}, Fábio Vital\textsuperscript{1,2}, Alberto Sardinha\textsuperscript{1,3}, Francisco S. Melo\textsuperscript{1,2}}

\emails{\{pedro.pinto.santos, fabiovital\}@tecnico.ulisboa.pt \ sardinha@inf.puc-rio.br, fmelo@inesc-id.pt}

\affiliations{
$^{1}$\textbf{INESC-ID}\\
$^{2}$\textbf{Instituto Superior Técnico, University of Lisbon}\\
$^{3}$\textbf{Pontifical Catholic University of Rio de Janeiro}
}

\contribution{
    We motivate, propose and formalize risk-aware GUMDPs, which enable agents and decision makers to trade off expected performance and risk aversion while benefiting from the rich set of objectives that can be cast under the framework of GUMDPs.
    }
    {
    Previous works studied risk-neutral (expected performance) policy optimization in GUMDPs \citep{zahavy_2021,mutti_2023,santos_2025}.
    }

\contribution{
    We show how we can solve risk-aware GUMDPs by resorting to online planning techniques, proposing an MCTS-based approach to provably solve risk-aware GUMDPs up to any desired accuracy.
    }
    {
    None.
    }

\contribution{
    We provide a set of experimental results showcasing that our approach is successful when optimizing for a spectrum of risk-aware behaviors in the context of GUMDPs under diverse tasks (standard MDPs, maximum state entropy exploration, imitation learning, and multi-objective MDPs).
    }
    {
    None.
    }


\summary{
We study general-utility Markov decision processes (GUMDPs) with risk-aware objectives. In this framework, an agent aims to optimize a risk measure of the distribution of objective values, where the objective function depends on the frequency of visitation of states induced by the agent's policy. First, we motivate, propose and formalize risk-aware GUMDPs, which enable agents and decision makers to trade off expected performance and risk aversion while benefiting from the rich set of objectives that can be cast under the framework of GUMDPs. We focus our attention to the entropic risk measure (ERM). Second, we show how we can solve risk-aware GUMDPs with ERM objectives by resorting to online planning techniques. In particular, we propose an MCTS-based approach to provably solve risk-aware GUMDPs up to any desired accuracy. Third, we provide a set of experimental results showcasing that our approach is successful when optimizing for a spectrum of risk-aware behaviors in the context of GUMDPs under diverse tasks (standard MDPs, maximum state entropy exploration, imitation learning, and multi-objective MDPs).
}

\begin{document}

\makeCover  
\maketitle  

\begin{abstract}
    We study general-utility Markov decision processes (GUMDPs) with risk-aware objectives. In this framework, an agent aims to optimize a risk measure of the distribution of objective values, where the objective function depends on the frequency of visitation of states induced by the agent's policy. First, we motivate, propose, and formalize risk-aware GUMDPs, which enable agents and decision makers to trade off expected performance by risk aversion while benefiting from the rich set of objectives that can be cast under the framework of GUMDPs. We focus our attention on the entropic risk measure (ERM). Second, we show how we can solve risk-aware GUMDPs with ERM objectives by resorting to online planning techniques. In particular, we propose an approach based on Monte Carlo Tree Search (MCTS) to provably solve risk-aware GUMDPs up to any desired accuracy. Third, we provide a set of experimental results showcasing that our approach is successful when optimizing for a spectrum of risk-aware behaviors in the context of GUMDPs under diverse tasks (standard MDPs, maximum state entropy exploration, imitation learning, and multi-objective MDPs). Code available at \href{https://github.com/gh0stwin/risk-aware-gumdp}{\texttt{https://github.com/gh0stwin/risk-aware-gumdp}}.
\end{abstract}

\section{Introduction} \label{sec:intro}
Markov decision processes (MDPs) \citep{puterman2014markov} have found a wide range of applications in different domains, ranging from optimal stopping \citep{chow1971great}, inventory management \citep{dvoretzky_1952}, or queueing control \citep{stidham_1978}. MDPs also play a central role in reinforcement learning (RL) \citep{sutton_1998}, where the agent-environment interaction is typically modelled within the MDP framework. In recent years, RL has achieved remarkable success across diverse domains \citep{mnih_2015,silver_2017,lillicrap_2016}.

However, many relevant objectives cannot be easily expressed within the standard MDP framework \citep{abel2022expressivity}. Examples include imitation learning \citep{hussein_2017,Osa_2018}, pure exploration problems \citep{hazan_2019}, risk-averse RL \citep{garcia_2015}, diverse skills discovery \citep{eysenbach2018diversity,achiam2018variational}, and constrained MDPs \citep{altman_1999,efroni_2020}. To address these limitations, the framework of general-utility Markov decision processes (GUMDPs) has been proposed as a more expressive formalism capable of modelling such objectives \citep{santos_2024}. In GUMDPs, the agent's objective is encoded as a function of the occupancy induced by a given policy, i.e., as a function of the frequency of visitation of state-action pairs induced by a given policy. GUMDPs generalize MDPs by allowing the objective function to be a non-linear function of occupancies.

Previous studies have been focused on studying policy optimization in GUMDPs by considering risk-neutral objectives, where the performance of a given policy is evaluated in terms of its induced \textit{expected} frequency of visitation of states \citep{mutti_2023,santos_2025,zahavy_2021,geist2022concave}. However, these approaches overlook the \textit{variability} in the objective: for a fixed policy, the stochastic nature of the environment can lead to a distribution over possible state-action visitation frequencies, resulting in a range of potential objective values. Consequently, optimizing only for the expectation may fail to capture important aspects of the policy's performance. Given the stochasticity of the possible outcomes, it is thus natural to consider optimizing a specific \textit{risk measure} \citep{artzner_1999} over the distribution of objective values, enabling the learned behavior to be tuned toward risk-seeking or risk-averse preferences. The \textit{entropic risk measure} (ERM) \citep{howard_matheson_1972} is one of the most common risk measures, which has received considerable attention across various domains such as finance and sequential decision-making \citep{follmer_2016,borkar_2002,pagnoncelli_2022,hau_2023,marthe2025efficientrisksensitiveplanningentropic,mortensen2025entropicriskoptimizationdiscounted}. As a risk-aware criterion, the ERM enables agents/decision-makers to trade off between expected performance and risk aversion, providing robustness against adverse (worst-case) outcomes. To further motivate the relevance of trading off between risk-neutral and risk-averse behaviors in the context of GUMDPs, we present the following illustrative example.

\begin{figure*}[t]
    \centering
    \begin{minipage}{0.4\textwidth}
        \begin{subfigure}[t]{\textwidth}
            \centering
            \includegraphics[height=1.2cm]{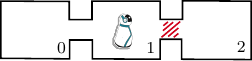}
            \caption{Environment illustration.}
            \label{fig:motivating_example:a}
        \end{subfigure}
        \vspace{4mm} 
        \begin{subfigure}[t]{\textwidth}
            \centering
            \includegraphics[height=2.5cm]{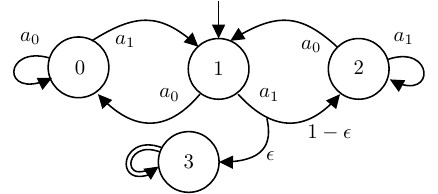}
            \caption{Abstracted environment.}
            \label{fig:motivating_example:b}
        \end{subfigure}
    \end{minipage}%
    \hfill
    \begin{minipage}{0.6\textwidth} 
        \vspace{0.1cm}
        \begin{subfigure}[t]{\textwidth}
            \centering
            \includegraphics[width=0.99\columnwidth]{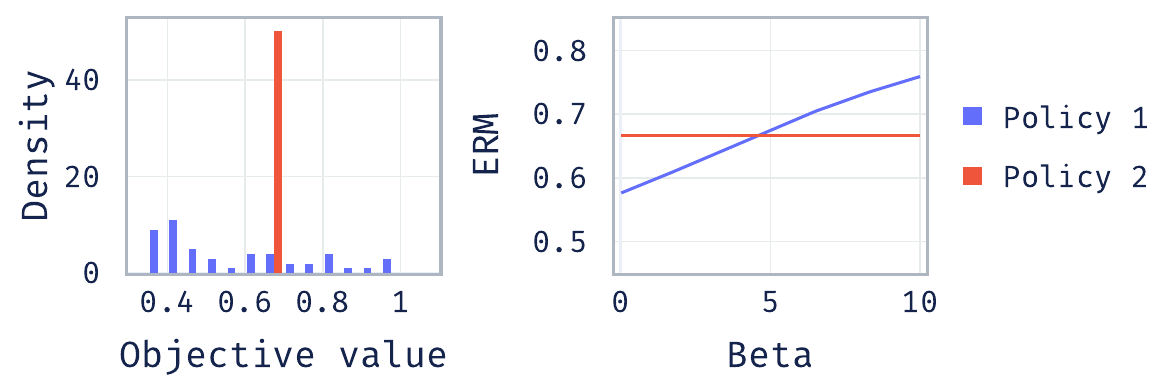}
            \caption{Objective function density (left) and respective $\text{ERM}_\beta$ values (right) for two policies $\pi_1$ (Policy 1) and $\pi_2$ (Policy 2). Lower values are better.}
            \label{fig:motivating_example:c}
        \end{subfigure}
    \end{minipage}
    \caption{Motivating example: Exploring three rooms.}
    \label{fig:motivating_example}
\end{figure*}

\paragraph{Motivating example}
We consider the illustrative environment in~\Cref{fig:motivating_example:a}, where a robot aims to explore three rooms as uniformly as possible (pure exploration task \citep{hazan_2019}). Transitioning from room $1$ to room $0$ incurs no danger to the robot. However, attempting to transition from room $1$ to room $2$ may cause irreparable damage to the robot, rendering it unable to explore the environment further. The dynamics of this environment can be abstracted by the MDP in~\Cref{fig:motivating_example:b}, where each state corresponds to one of the rooms, the actions encode the transitions between rooms, and $\epsilon \in (0,1)$ is the probability of irreparable damage to the robot when attempting to transition between rooms $1$ and $2$. We can resort to a GUMDP to encode the agent's objective by setting the objective function to be (minus) the entropy of the occupancy induced by any policy. Now, a \textit{risk-averse} agent will not be interested in trying to explore state 2 since there exists a probability that it will get absorbed into state 3 and, thus, not being able to explore the environment further. On the other hand, a \textit{risk-seeker} agent will try to visit state 2 since it may be successful, hence exploring all three rooms. In~\Cref{fig:motivating_example:c}, we compare two policies:
\begin{enumerate*}[label=(\roman*)]
    \item $\pi_1$, a \textit{risk-seeking} policy that attempts to visit state 2; and
    \item $\pi_2$, a risk-averse policy that does not attempt to visit state 2 and, instead, keeps alternating between states 0 and 1.
\end{enumerate*}
The left plot shows the density of the objective values for the policies. Both policies yield very different distributions, with the risk-seeking policy $\pi_1$ attaining lower (better) objective values at the cost of sometimes obtaining higher (worse) objective values. On the other hand, the risk-averse policy $\pi_2$ always attains the same intermediate objective value due to its deterministic behavior. The plot on the right shows the $\text{ERM}_\beta$ of both policies as a function of $\beta \in (0, \infty)$. Higher $\beta$ values correspond to more risk-averse preferences, while $\beta$ approaching zero reflects risk neutrality. As shown, $\pi_2$ achieves a lower $\text{ERM}_\beta$ value for higher $\beta$, confirming its risk-averse behavior. Conversely, as $\beta$ decreases, $\pi_1$ becomes the superior policy.

\paragraph{Contributions}
We study risk-aware GUMDPs with ERM objectives. As a risk-aware criterion, the ERM enables agents/decision-makers to trade off between expected performance and risk aversion, providing robustness against adverse (worst-case) outcomes. Our \textit{contributions} are threefold. \textit{First}, we motivate, propose, and formalize risk-aware GUMDPs, which enable agents to trade off between expected performance and risk aversion while benefiting from the rich set of objectives that can be cast under the framework of GUMDPs. We focus our attention on the ERM, as it has been commonly used by previous works in the context of MDPs \citep{borkar_2002,pagnoncelli_2022,hau_2023,marthe2025efficientrisksensitiveplanningentropic,mortensen2025entropicriskoptimizationdiscounted}. \textit{Second}, we show how we can solve risk-aware GUMDPs with ERM objectives (ERM-GUMDP) by resorting to online planning techniques. In particular, we show that any ERM-GUMDP can be reformulated as a specific MDP — referred to as an occupancy MDP — with an ERM objective. This reformulation enables the use of online planning techniques to compute approximately optimal policies for arbitrary ERM-GUMDPs. We propose an MCTS-based approach to provably solve risk-aware GUMDPs up to any desired accuracy. \textit{Third}, we provide experimental results showing that our approach successfully optimizes for a spectrum of risk-aware behaviors in the context of GUMDPs under diverse tasks (standard MDPs, maximum state entropy exploration, imitation learning, and multi-objective MDPs).

\section{Background}
\label{sec:background}

\subsection{Markov decision processes} \label{sec:back_mdp}
MDPs \citep{puterman2014markov} can be defined by the tuple $\mathcal{M} = (\mathcal{S}, \mathcal{A}, \{\mat{P}^a : a \in \mathcal{A} \}, \vec{p_0}, c)$ where: $\mathcal{S}$ is the discrete state space; $\mathcal{A}$ is the discrete action space; $\{\mat{P}^a : a \in \mathcal{A} \}$ is a set of transition probability matrices $\mat{P}^a$ for each action $a \in \mathcal{A}$; $\vec{p_0} \in \Delta(\mathcal{S})$ is the distribution of initial states; and $c: \mathcal{S} \times \mathcal{A} \rightarrow \mathbb{R}$ is a cost function. The interaction between the agent and its environment, as modeled by the MDP framework, is as follows:
\begin{enumerate*}[label=(\roman*)]
    \item an initial state $\rvar{s}_0$ is sampled from $\vec{p}_0$;
    \item at each step $t$, the agent observes the state of the environment $\rvar{s}_t \in \mathcal{S}$ and chooses an action $\rvar{a}_t \in \mathcal{A}$.
\end{enumerate*}
Depending on the action chosen by the agent, the MDP evolves to a new state $\rvar{s}_{t+1} \in \mathcal{S}$ with probability given by $\mat{P}^{\rvar{a}_t}(\rvar{s}_t, \cdot)$, and the agent receives a cost $c(\rvar{s}_t, \rvar{a}_t)$; and (iii) the interaction repeats infinitely.

A decision rule $\pi_t$ specifies the action-selection mechanism at timestep $t$. A non-Markovian decision rule maps the history of states and actions to a probability distribution over actions, i.e., $\pi_t : \mathcal{S} \times (\mathcal{S} \times \mathcal{A})^t \rightarrow \Delta(\mathcal{A})$. In contrast, a Markovian decision rule depends only on the current state and maps it to a distribution over actions, i.e., $\pi_t : \mathcal{S} \rightarrow \Delta(\mathcal{A})$. Both types of decision rules can be deterministic if they map directly to actions rather than distributions: for non-Markovian rules, $\pi_t : \mathcal{S} \times (\mathcal{S} \times \mathcal{A})^t \rightarrow \mathcal{A}$, and for Markovian rules, $\pi_t : \mathcal{S} \rightarrow \mathcal{A}$. A policy $\pi = (\pi_0, \pi_1, \ldots)$ is a sequence of decision rules, one for each timestep. A policy is Markovian or non-Markovian if all its decision rules are Markovian or non-Markovian, respectively. Likewise, a policy is deterministic or stochastic if all its decision rules are deterministic or stochastic. We denote the set of non-Markovian policies by $\Pi_{\text{NM}}$, Markovian policies by $\Pi_{\text{M}}$, non-Markovian deterministic policies by $\Pi^\text{D}_{\text{NM}}$, and Markovian deterministic policies by $\Pi^\text{D}_{\text{M}}$. A policy is stationary if it uses the same decision rule at every timestep. We denote the set of stationary policies by $\Pi_{\text{S}}$ and the set of stationary deterministic policies by $\Pi_{\text{S}}^\text{D}$. Among the relationships between these classes, we highlight the following:
\begin{enumerate*}[label=(\roman*)]
    \item $\Pi_{\text{NM}}$ is the most general class;
    \item $\Pi_{\text{S}} \subset \Pi_{\text{M}} \subset \Pi_{\text{NM}}$; and
    \item $\Pi_{\text{S}}^\text{D} \subset \Pi_{\text{M}}^\text{D} \subset \Pi_{\text{NM}}^\text{D}$.
\end{enumerate*}

For a given policy $\pi \in \Pi_\text{NM}$, the interaction between the agent and the environment induces a random process $(\rvar{s}_0, \rvar{a}_0, \rvar{s}_1, \rvar{a}_1, \ldots)$. We denote by $\rvar{h}_t = (\rvar{s}_0, \rvar{a}_0, \rvar{s}_1, \rvar{a}_1, \ldots, \rvar{s}_t)$ the random history up to and including timestep $t$, and by $h_t = (s_0, a_0, s_1, a_1, \ldots, s_t) \in \S \times (\S \times \A)^t$ a particular realization of such a history. The random process $(\rvar{s}_0, \rvar{a}_0, \rvar{s}_1, \rvar{a}_1, \ldots)$ satisfies the following conditions:
\begin{enumerate*}[label=(\roman*)]
    \item $\mathbb{P}[\rvar{s}_0 = s] = p_0(s)$;
    \item $\mathbb{P}[\rvar{s}_{t+1} = s' \mid \rvar{h}_t, \rvar{a}_t] = P^{\rvar{a}_t}(\rvar{s}_t, s')$; and
    \item $\mathbb{P}[\rvar{a}_t = a \mid \rvar{h}_t] = \pi(a \mid \rvar{h}_t)$.
\end{enumerate*}
Let $(\Omega, \F, \mathbb{P}_\pi)$ be the probability space over trajectories $(\rvar{s}_0, \rvar{a}_0, \rvar{s}_1, \rvar{a}_1, \ldots)$ satisfying (i)–(iii), as defined in \citet{lattimore_2020}. We denote a sample trajectory by $\omega \in \Omega$, where $\omega = (s_0, a_0, s_1, a_1, \ldots)$. We also let $\mathbb{P}_\pi[\rvar{s}_t = s, \rvar{a}_t = a]$ be the probability of observing the state-action pair $(s, a)$ at timestep $t$ under the MDP while following policy $\pi$.

The expected discounted cumulative cost objective is
$J_{\gamma}^\pi = \mathbb{E}\left[\sum_{t=0}^\infty \gamma^t c(\rvar{s}_t, \rvar{a}_t) \right],$
where $\gamma \in (0,1)$ is the discount factor and the expectation is taken over the random trajectory of state-action pairs. The discounted state-action occupancy for policy $\pi$ is defined as
\begin{align}
\label{eq:discounted_state_action_occupancy}
d_{\pi}(s,a) &= (1-\gamma) \sum_{t=0}^\infty \gamma^t \mathbb{P}_\pi\left[\rvar{s}_t = s, \rvar{a}_t = a\right].
\end{align}
The expected discounted cumulative cost of policy $\pi$ can be rewritten as
$J_{\gamma}^\pi = \vec{c}^\top \vec{d}_{\pi},$
where
$\vec{d}_{\pi} = [d_{\pi}(s_0,a_0), \ldots, d_{\pi}(s_{|\mathcal{S}|},a_{|\mathcal{A}|})]^\top, \; \vec{c} = [c(s_0,a_0), \ldots, c(s_{|\mathcal{S}|},a_{|\mathcal{A}|})]^\top.$
We aim to find
$\pi^* = \argmin_{\pi \in \Pi_\text{S}} J_{\gamma}^\pi$, 
which can be formulated as a linear program \citep{puterman2014markov}.

\subsection{Risk-aware MDPs}
\label{sec:background:risk_aware_mdps}
Risk-aware MDPs \citep{NIPS2015_64223ccf} generalize the classical MDP framework by considering the associated variability of the discounted cumulative costs. This is done by considering
$J_{\gamma, \rho}^\pi = \rho\left(\sum_{t=0}^\infty \gamma^t c(\rvar{s}_t, \rvar{a}_t) \right),$
where $\rho$ is a \textit{risk measure} \citep{artzner_1999}, computed over the random trajectory.
Risk measures quantify an agent's attitude toward uncertainty and variability in outcomes. We note that for $\rho(\cdot)=\mathbb{E}[\cdot]$ we have the standard discounted MDP objective. The risk-aware objective allows the agent to reason over the entire distribution instead of just its expectation.

In this work, we consider risk-aware MDPs with an ERM objective \citep{howard_matheson_1972}. The ERM of a random variable $\rvar{z}$ is defined as
$\text{ERM}_\beta(\rvar{z}) = \frac{1}{\beta} \ln\left( \mathbb{E}\left[\exp(\beta \rvar{z})\right]\right)$,
where parameter $\beta \in (0, \infty)$ trades off between expected performance as $\beta \rightarrow 0$ and risk-aversion as $\beta$ increases. Also, the $\text{ERM}_\beta$ belongs to the family of optimized certainty equivalents (OCEs) \citep{bental_1986} since it can be equivalently defined as 
$\text{ERM}_\beta(\rvar{z}) = \min_{\lambda \in \mathbb{R}} \left\{ \lambda + \mathbb{E}\left[ u(\rvar{z} - \lambda) \right]\right\},$
where $u(x) = \frac{1}{\beta} \left( \exp(\beta x) - 1\right)$. In this work, we focus our attention on the $\text{ERM}_\beta$ and, thus, we implicitly refer to the $\text{ERM}_\beta$ whenever we write $\rho$.

\subsection{General-utility MDPs}
\label{sec:background:gumpds}
GUMDPs generalize utility specification by allowing the agent's objective to be written in terms of the frequency of visitation of state-action pairs. This is in contrast to the standard MDPs framework, where the objective of the agent is encoded by the cost function. 

We define an infinite-horizon discounted GUMDP as a tuple $\mathcal{M}_f = (\mathcal{S}, \mathcal{A}, \{\mat{P}^a : a \in \mathcal{A} \}, \vec{p_0}, f)$, where $\mathcal{S}$, $\mathcal{A}$, $\{\mat{P}^a : a \in \mathcal{A} \}$, and $\vec{p_0}$ are defined in a similar way to the standard MDP formulation, as introduced in \Cref{sec:back_mdp}. The function $f: \Delta(\mathcal{S} \times \mathcal{A}) \rightarrow \mathbb{R}$ encodes the objective of the agent, which depends on a discounted occupancy $\vec{d}$, as defined in \eqref{eq:discounted_state_action_occupancy}. We then aim to find
$\pi^* \in \argmin_{\pi \in \Pi_\text{S}} f(\vec{d}_\pi).$
If $f$ is a linear function, then we are under the standard MDP setting. If $f$ is convex, then we are under the convex MDP/RL setting \citep{zahavy_2021}.

In this work, we consider four tasks, each associated with a particular objective function:
\begin{enumerate*}[label=(\roman*)]
    \item cost minimization (standard MDP) \citep{puterman2014markov}, where $f(\vec{d}) = \vec{c}^\top \vec{d}$ and $\vec{c} \in \mathbb{R}^{|\S||\A|}$ is the column vector encoding the cost function;
    \item maximum state entropy exploration (MSEE) \citep{hazan_2019}, where $f(\vec{d}) = \vec{d}^\top\log(\vec{d})$, i.e., we aim to find a policy that visits all state-action pairs as uniformly as possible;
    \item imitation learning (IL) \citep{abbeel_2004}, where $f(\vec{d}) = \| \vec{d} - \vec{d}_{\pi_b}\|_2^2$ with $\vec{d}_{\pi_b} \in \Delta(\S \times \A)$, i.e., we aim to find a policy such that its induced occupancy is as similar as possible to the occupancy $\vec{d}_{\pi_b}$ induced by a given behavior policy $\pi_b$; and
    \item multi-objective (MO) MDPs, where $f(d) = g(\vec{d}^\top \vec{c}_1, \ldots, \vec{d}^\top \vec{c}_k)$, where $\vec{c}_1, \ldots, \vec{c}_k \in \mathbb{R}^{|\S||\A|}$ are vectorized cost functions and $g$ is an utility function \citep{radulescu_2019}.
\end{enumerate*}
Nevertheless, we highlight that our results apply to any task that can be modeled using the GUMDPs framework. We refer to \cite{zahavy_2021} for a comprehensive list of the different objectives considered by previous works.

\section{Towards Risk-Aware GUMDPs}
We now introduce the framework of risk-aware GUMDPs. Because risk-awareness requires reasoning about the underlying distribution of objective values, we begin by adopting a distributional perspective on GUMDPs. Similarly to distributional RL \citep{bdr2023}, we employ the following assumption to ensure that all random variables defined below are bounded and well-defined.  

\begin{assumption}
    \label{assumption:bounded_objective}
    The objective $f$ satisfies $f(\vec{d}) \in \Y = [-R,R]$, $\forall \vec{d} \in \Delta(\S \times \A)$, where $R \in \mathbb{R}^+$.
\end{assumption}

\subsection{A distributional perspective on GUMDPs}
For a given fixed policy $\pi \in \Pi_\text{NM}$, let $\rvec{d}^\pi : \Omega \rightarrow \Delta(\S \times \A)$ be the random vector defined on the probability space $(\Omega, \F, \mathbb{P}_\pi)$ with entries
\begin{equation}
    \rvar{d}^\pi_{s,a}(\omega) = (1-\gamma) \sum_{t=0}^\infty \gamma^t \mathbf{1}(s_t=s,a_t=a). \label{eq:empirical_occupancy}
\end{equation}

The random vector above is well-defined since the mapping $\rvec{d}^\pi$ is $\F/\B(\Delta(\S \times \A))$ measurable, where $\B(\Delta(\S \times \A))$ denotes the Borel $\sigma$-algebra of $\Delta(\S \times \A)$. Let also $\rvar{f}^\pi : \Omega \rightarrow \Y$ be the random variable with support in $\Y$ such that
$$\rvar{f}^\pi(\omega) = (f \circ \rvec{d}^\pi) (\omega) = f(\rvec{d}^\pi(\omega)),$$
satisfying
$$\mathbb{P}\left[ \rvar{f}^\pi \le \eta \right] = \mathbb{P}_\pi\left[\left\{\omega \in \Omega : f(\rvec{d}^\pi(\omega)) \le \eta \right\}\right], \quad \forall \eta \in \Y,$$
where we assume $f$ is a Borel measurable function. Essentially, the random variable $\rvar{f}^\pi$ is defined as the composition of two measurable mappings. If $f$ is continuous, then $f$ is Borel measurable.

We refer to~\Cref{fig:motivating_example:c} for an illustration of the density associated with the random variable $\rvar{f}^\pi$ for two arbitrary policies. As observed, the underlying distributions induced by different policies can differ substantially and, therefore, depending on the exact preferences of the decision-maker, one policy may be preferable over the other. Risk measures, such as the $\text{ERM}_\beta$, allow for the comparison between distributions induced by different policies.

\subsection{Risk-aware GUMDPs}
In the context of GUMDPs, we denote $\L$ the set of possible random variables with support in $\Y$ associated with each of the probability spaces $(\Omega, \F, \mathbb{P}_\pi)$, for all $\pi \in \Pi_\text{NM}$. Note that $\rvar{f}^\pi \in \L$ for any $\pi \in \Pi_\text{NM}$. Then, we consider risk measures $\rho: \L \rightarrow\mathbb{R}$ that associate any random variable in $\L$ with a scalar value. In this work, we are particularly interested in the $\text{ERM}_\beta$, as introduced in~\cref{sec:background:risk_aware_mdps}. We aim to find 
\begin{equation}
    \pi^* = \argmin_{\pi \in \Pi_\text{NM}} \rho(\rvar{f}^\pi). \label{eq:risk_aware_gumdp_objective}
\end{equation}
For any policy $\pi \in \Pi_\text{NM}$, we define its optimality gap as
$\mathrm{OptGap}(\pi) = \rho\left(\rvar{f}^\pi\right) -  \min_{\pi' \in \Pi_{\text{NM}}} \rho\left(\rvar{f}^{\pi'}\right).$
Essentially, $\mathrm{OptGap}(\pi)$ measures how suboptimal a given policy $\pi$ is compared to the best policy.

\paragraph{Policy optimization in risk-aware GUMDPs}
While in the context of risk-aware MDPs with ERM objectives the class of deterministic and Markovian policies, $\Pi_\text{M}^\text{D}$, suffices for optimality \citep{hau_2023}, this is no longer the case in the context of risk-aware GUMDPs with ERM objectives. As shown in \cite{santos_2025} in the context of GUMDPs, if $\rho(\cdot) = \mathbb{E}[\cdot]$ then the class of Markovian policies is strictly dominated by the class of non-Markovian policies. Since, in the context of risk-aware GUMDPs, as $\beta \rightarrow 0$ the $\text{ERM}_\beta$ approaches $\mathbb{E}[\cdot]$, non-Markovian policies are also needed, in general, in the context of solving GUMDPs with ERM objectives.

\section{Solving Risk-Aware GUMDPs with Online Planning}
\label{sec:solving_GUMDPs_with_online_planning}
In this section, we study how to solve risk-aware GUMDPs using online planning techniques. We begin by establishing a close connection between risk-aware GUMDPs and a particular risk-aware MDP in which the agent tracks the accrued occupancy at each timestep. Leveraging this formulation, we show that solving the risk-aware MDP with online planning techniques yields approximately optimal action selection for the original risk-aware GUMDP.

\begin{assumption}
    \label{assumption:lipschitz_objective}
    The objective function $f$ is $L_f$-Lipschitz, i.e., $|f(\vec{d_1}) - f(\vec{d_2}) | \le L_f \| \vec{d}_1 - \vec{d}_2 \|_1$, for any $\vec{d}_1, \vec{d}_2 \in \Delta(\S \times \A)$.
\end{assumption}

We refer to~\Cref{appendix:A} for the derivation of $L_f$-Lipschitz constants for the objective functions considered.

\subsection{Computing optimal policies by resorting to finite-horizon GUMDPs}
We start by showing that we can compute an approximately optimal policy for our infinite-horizon discounted risk-aware objective \eqref{eq:risk_aware_gumdp_objective} by resorting to a finite-horizon formulation of the same problem. For any policy $\pi \in \Pi_\text{NM}$, let $\rvar{f}^\pi_H : \Omega \rightarrow \Y$ be the random variable with support in $\Y$ such that
\begin{align}
    \rvar{f}^\pi_H(\omega) &= (f \circ \rvec{d}^{\pi,H}) (\omega) = f(\rvec{d}^{\pi,H}(\omega)) \label{eq:truncated_f_random_var_definition} \\
    \rvar{d}^{\pi,H}_{s,a}(\omega) &= \frac{1-\gamma}{1-\gamma^H} \sum_{t=0}^{H-1} \gamma^t \mathbf{1}(s_t=s,a_t=a)\label{eq:empirical_truncated_occupancy}.
\end{align}
Essentially, random variable $\rvar{f}^\pi_H$ is defined in a similar fashion to the random variable $\rvar{f}^\pi$, but the former considers a random occupancy $\rvec{d}^{\pi,H}$ that only takes into account the state-action pairs visited up to a given horizon $H \in \mathbb{N}$. We state the following result (proof in~\Cref{appendix:B}).

\begin{proposition}[Optimality gap decomposition] \label{proposition:opt_gap_decomposition}
    For arbitrary $\pi \in \Pi_\textnormal{NM}$, it holds that
    \begin{equation*}
        \mathrm{OptGap}(\pi) \le \underbrace{ \rho(\rvar{f}^\pi_H)  - \min_{\pi' \in \Pi_\text{NM}} \left\{ \rho(\rvar{f}^{\pi'}_H) \right\} }_{\text{$ = \mathrm{OptGap}_H(\pi)$}} \; + \; 8 L_f \gamma^H,
    \end{equation*}
    where $\mathrm{OptGap}_H(\pi)$ is the optimality gap of policy $\pi$ for the truncated objective with horizon $H$.
\end{proposition}
The result above shows that we can resort to the truncated objective to compute approximately optimal policies, up to any desired accuracy, for the infinite-horizon risk-aware objective \eqref{eq:risk_aware_gumdp_objective}. Therefore, we focus our attention on solving
$\pi^* = \argmin_{\pi \in \Pi_\text{NM}} \rho(\rvar{f}^\pi_H),$
where $\rvar{f}^\pi_H$ is defined in \eqref{eq:truncated_f_random_var_definition}. Since we now focus on the case of finite-horizon trajectories, we assume that, for a given policy $\pi \in \Pi_\text{NM}$, the interaction between the agent and the environment gives rise to a random process $(\rvar{s}_0, \rvar{a}_0, \rvar{s}_1, \rvar{a}_1, \ldots, \rvar{s}_{H-1}, \rvar{a}_{H-1})$ such that:
\begin{enumerate*}[label=(\roman*)]
    \item $\PP{\rvar{s}_0 = s} = p_0(s)$;
    \item $\PP{\rvar{s}_{t+1} = s' | \rvar{h}_{t}, \rvar{a}_t} = P^{\rvar{a}_t}(\rvar{s}_{t}, s')$; and
    \item $\PP[]{\rvar{a}_t = a | \rvar{h}_{t}} = \pi(a|\rvar{h}_{t})$.
\end{enumerate*}
We thus let from now on $(\Omega, \F, \mathbb{P}_\pi)$ be the probability space over the sequence of random variables $(\rvar{s}_0, \rvar{a}_0, \rvar{s}_1, \rvar{a}_1, \ldots, \rvar{s}_{H-1}, \rvar{a}_{H-1})$ that satisfies conditions (i)-(iii) above. We write specific trajectories as $\omega \in \Omega$, with $\omega = (s_0, a_0, s_1, a_1, \ldots, s_{H-1}, a_{H-1})$. We highlight that the probability of a given trajectory $\omega \in \Omega$ under policy $\pi \in \Pi_\text{NM}$ can be calculated as $\PP[\pi]{\omega} = p_0(s_0) \cdot \pi(a_0|h_0) \cdot P^{a_0}(s_0, s_1) \cdot \pi(a_1|h_1) \cdot P^{a_1}(s_1, s_2) \ldots P^{a_{H-2}}(s_{H-2}, s_{H-1}) \cdot \pi(a_{H-1}|h_{H-1})$.

\subsection{The occupancy MDP}
\label{sec:occupancy_mdp}
To derive our planning algorithms for solving risk-aware GUMDPs, we make use of a finite-horizon MDP derived from the original GUMDP formulation. In particular, we consider the occupancy MDP \citep{santos_2025}, defined by the tuple $\mathcal{M}_{\text{O}} = \{\mathcal{S}_\text{O}, \mathcal{A}_\text{O}, \{\mat{P}_\text{O}^a : a \in \mathcal{A} \}, \vec{p}_{0,\text{O}}, c_\text{O}, H\}$, where $\S_\text{O} = \mathcal{S} \times \O$ is the discrete state space and $\O$ is the set of occupancies up to length $H-1$ in the original GUMDP. We let $\{s,\vec{o}\}$ be a state of the occupancy MDP such that $s \in \S$ is a state from the original GUMDP and $\vec{o} \in \O$ is a $|\S||\A|$-dimensional vector that keeps track of the running occupancy of the agent up to a given timestep. Then, $\A_\text{O} = \A$ is the action space and $\vec{p}_{0,\text{O}}$ is such that $p_{0,\text{O}}(\{s,\vec{o}\}) = p_0(s)$ if $\vec{o} = [0,\ldots,0]$ and zero otherwise. The dynamics are as follows: (i) component $\rvar{s}_{t+1} \sim \mat{P}^{\rvar{a}_t}(\cdot | \rvar{s}_t)$ evolves according to the dynamics of the original GUMDP; and (ii) the running occupancy evolves deterministically as $o_{t+1}(s,a) = \gamma^t + o_t(s,a)$ if $s = \rvar{s}_t$ and $a = \rvar
{a}_t$, and $o_{t+1}(s,a) = o_t(s,a)$ otherwise. Finally, $H \in \mathbb{N}$ denotes the horizon and the cost function $c_\text{O} : \S \times \O \rightarrow \mathbb{R}$ is defined as 
\begin{equation*}
    c_\text{O}(\{s,\vec{o}\}) = \begin{cases}
    0 & \text{if $t < H $}, \\
    f\left(\frac{1-\gamma}{1-\gamma^H} \vec{o}\right) & \text{if $t = H $}.
  \end{cases}
\end{equation*}

Stationary policies $ \pi_\text{O} \in \Pi_\text{S}$ for $\M_\text{O}$ are mappings of the type $\pi_\text{O} : \S \times \O \rightarrow \Delta(\A)$. For a given policy $\pi_\text{O} \in \Pi_\text{S}$, the interaction between the agent and the occupancy MDP gives rise to a random process $(\{\rvar{s}_0,\rvec{o}_0\}, \rvar{a}_0, \{\rvar{s}_1,\rvec{o}_1\}, \rvar{a}_1, \ldots, \{\rvar{s}_{H}, \rvec{o}_H\})$ associated with the probability space  $(\Omega_\text{O}, \F_\text{O}, \mathbb{P}_{\pi_\text{O}}^O)$. We write specific trajectories as $\omega_\text{O} \in \Omega_\text{O}$, with $\omega_\text{O} = (\{s_0,\vec{o}_0\}, a_0, \{s_1,\vec{o}_1\}, a_1, \ldots, \{s_{H}, \vec{o}_H\})$.

For any policy $ \pi_\text{O} \in \Pi_\text{S}$, we let $\rvar{J}_\text{O}^{\pi_\text{O}}: \Omega_\text{O} \rightarrow \Y$ be the random variable with support in $\Y$ associated with the probability space $(\Omega_\text{O}, \F_\text{O}, \mathbb{P}_{\pi_\text{O}}^\text{O})$ such that
$
    \rvar{J}_\text{O}^{\pi_\text{O}}(\omega_\text{O}) = \sum_{t=0}^H c_\text{O}(\{s_t, \vec{o}_t\}).
$
We reproduce the following result from \citet{santos_2025}.

\begin{lemma}[One-to-one mapping between histories in $\mathcal{M}_f$ and states in $\mathcal{M}_\text{O}$]\label{lemma:one_to_one_mapping_histories_states}
    There exists a one-to-one mapping between histories $h_l = (s_0,a_0,s_1,a_1, \ldots, s_l) \in \S \times (\S \times \A)^l$ in $\M_f$, with $0 \le l \le H-1$, and states $\{s,\vec{o}\} \in \S \times \O$ in $\M_\text{O}$.
\end{lemma}

An important conclusion that can be derived from the result above is that there exists a one-to-one mapping between non-Markovian policies for $\M_f$ and stationary policies for $\M_\text{O}$. This is because every state in $\M_\text{O}$ is uniquely associated with a particular history in $\M_f$ (and vice versa), as the result above shows. With this in mind, we now state the following results (proofs in~\Cref{appendix:B}).

\begin{lemma}[Equivalence in distribution between $\rvar{f}_H^\pi$ and $\rvar{J}^{\pi_\text{O}}_\text{O}$] \label{lemma:equivalence_in_distribution}
    For any horizon $H \in \mathbb{N}$ and policy $\pi \in \Pi_\text{NM}$, it holds that $\rvar{f}_H^\pi \overset{D}{=} \rvar{J}_\text{O}^{\pi_\text{O}}$, i.e., random variables $\rvar{f}_H^\pi$ and $\rvar{J}_\text{O}^{\pi_\text{O}}$ are equal in distribution, where $\pi_\text{O}$ is the stationary policy for $\M_\text{O}$ associated with the non-Markovian policy $\pi$ for $\M_f$.
\end{lemma}

\begin{theorem}[Solving the risk-aware $\M_f$ is ``equivalent'' to solving the risk-aware $\M_\text{O}$] \label{theo:gumdp_equivalence_occupancy_MDP}
    For a given risk measure $\rho : \L \rightarrow \mathbb{R}$, the problem of finding a policy $\pi \in \Pi_\text{NM}$ satisfying $\mathrm{OptGap}_H(\pi) \le \epsilon$, for any $\epsilon \in \mathbb{R}_0^+$, can be reduced to the problem of finding a policy $\pi_\text{O} \in \Pi_\text{S}$ satisfying
    $$ \rho\left(\rvar{J}_\text{O}^{\pi_\text{O}}\right) - \min_{\pi'_\text{O} \in \Pi_\text{S}} \rho\left(\rvar{J}_\text{O}^{\pi'_\text{O}}\right) \le \epsilon.$$
    In particular, if $\pi^*_\text{O} = \argmin_{\pi_\text{O} \in \Pi_\text{S}} \rho(\rvar{J}_\text{O}^{\pi_\text{O}})$, then the corresponding non-Markovian policy $\pi$ in $\M_f$ satisfies $\mathrm{OptGap}_H(\pi) = 0$. It also holds that
    $\mathrm{OptGap}_H(\pi) = \rho\left(\rvar{J}_\text{O}^{\pi_\text{O}}\right) - \min_{\pi'_\text{O} \in \Pi_\text{S}} \rho\left(\rvar{J}_\text{O}^{\pi'_\text{O}}\right),$
    where $\pi_\text{O}$ is the stationary policy for $\M_\text{O}$ associated with the non-Markovian policy $\pi$ for $\M_f$.
\end{theorem}

The result above shows that it suffices to search for an (approximately) stationary optimal policy for $\M_\text{O}$, since such a policy corresponds to a non-Markovian policy that is (approximately) optimal for $\M_f$. In particular, such an (approximately) optimal policy for $\M_\text{O}$ can be seen as a non-Markovian policy for $\M_f$ that compresses the history up to any timestep into a running occupancy. The result above shows that \textit{it suffices to keep track of the running occupancy up to any timestep in order to attain optimal behavior} when solving risk-aware GUMDPs.

In light of~\Cref{theo:gumdp_equivalence_occupancy_MDP}, we consider risk-aware planning algorithms to solve the occupancy MDP. Unfortunately, solving the risk-aware occupancy MDP poses some challenges. One of the key challenges is due to the fact that the size of the state space of the occupancy MDP grows combinatorially with $H$ since every state in the occupancy MDP is associated with a possible history in $\M_f$. Consequently, the state space of the occupancy MDP is typically very large, which precludes the use of offline planning methods. In the next section, we therefore turn to online planning approaches to implicitly compute approximately risk-aware optimal policies for the occupancy MDP.

\subsection{Solving risk-aware occupancy MDPs via online planning} \label{sec:on-planning}
We aim to propose a practical algorithm to compute an approximately optimal risk-aware policy for the occupancy MDP $\M_O$, as introduced in~\Cref{sec:occupancy_mdp}. Since the occupancy MDP is a standard undiscounted finite-horizon MDP, we can resort to the risk-aware Monte Carlo tree search (MCTS) algorithm put forth by \citet{santos2026entropicriskawaremontecarlo}, \textsc{ERM-MCTS}, to solve the risk-aware occupancy MDP with the ERM objective. \textsc{ERM-MCTS} works in a similar fashion to standard MCTS in the sense that the algorithm iteratively builds a search tree that alternates between decision nodes, where actions are selected, and chance nodes corresponding to random next states sampled from the MDP. At each iteration, \textsc{ERM-MCTS} refines the search tree by simulating a random trajectory in the occupancy MDP. Action-selection at each decision node $\rvar{s}_t \in \S_{O}$ is given by
\begin{equation*}
    \rvar{a}_t \in \argmin_{a \in \A} \Bigg\{ \frac{1}{\beta} \ln\left(\frac{1}{N(\rvar{s}_t,a)} \sum_{i=1}^{N(\rvar{s}_t,a)} \exp \left(\beta \rvar{x}_i^{(\rvar{s}_t,a)} \right)\right) - \theta_t \sqrt{\frac{\sqrt{N(\rvar{s}_t)}}{N(\rvar{s}_t,a)}}\Bigg\},
\end{equation*}
where $N(\rvar{s}_t)$ is the number of times $\rvar{s}_t$ has been visited, $N(\rvar{s}_t,a)$ the number of times action $a$ has been selected while in $\rvar{s}_t$, $\rvar{x}_i^{(\rvar{s}_t,a)}$ is the $i$-th sampled terminal cost starting from $(\rvar{s}_t,a)$ (given that the occupancy MDP only has non-zero costs at $t=H$), and $\theta_t$ is an exploration constant. In case $N(\rvar{s}_t,a) = 0$ for some $a \in \A$, then $a$ is selected. We refer to~\Cref{appendix:ERM-MCTS-pseudocode} for the full pseudocode of the \textsc{ERM-MCTS} algorithm. Without loss of generality, assume we fix an initial state $s_0 \in \S$. We state the following result, which is a consequence of Theo.~6 in \citet{santos2026entropicriskawaremontecarlo}.

\begin{theorem}
    There exist exploration constants $(\theta_0, \theta_1, \ldots, \theta_{H-1})$ such that, from any initial state $s_0 \in \S$, \textsc{ERM-MCTS} provably solves the risk-aware occupancy MDP with an ERM objective. More precisely, let $\hat{V}_n(s_0) = \frac{1}{\beta} \ln \left(\frac{1}{n} \sum_{a \in \A} \sum_{i=1}^{T_a(n)} \exp\left(\beta \rvar{x}_i^{(\rvar{s}_t,a)}\right)\right)$ be the empirical ERM obtained by the \textsc{ERM-MCTS} algorithm after $n$ iterations at initial state $s_0$, where $T_a(n)$ denotes the number of times action $a$ was selected at root node $s_0$ after $n$ iterations. Then,
    $$\lim_{n \rightarrow \infty} \mathbb{E} \left[ \hat{V}_n(s_0) \right] \overset{(a)}{=} \min_{\pi'_\text{O} \in \Pi_\text{S}} \rho\left(\rvar{J}_\text{O}^{\pi'_\text{O}}\right) \overset{(b)}{=} \min_{\pi' \in \Pi_\text{NM}} \left\{ \rho(\rvar{f}^{\pi'}_H) \right\},$$
    where: (a) follows from Theo.~6 in \citet{santos2026entropicriskawaremontecarlo} as the expected empirical ERM obtained by \textsc{ERM-MCTS} converges, in the limit, to the optimal ERM for the occupancy MDP; and (b) follows from Theo.~\ref{theo:gumdp_equivalence_occupancy_MDP}.
\end{theorem}

\section{Experimental Results}
\label{sec:experimental_results}
We empirically assess the performance of \textsc{ERM-MCTS} for solving risk-aware GUMDPs, investigating how \textsc{ERM-MCTS} trades off risk-neutral and risk-averse behavior. We describe our experimental methodology and refer to~\Cref{appendix:experimental-results} for a complete description of our experiments (environments, baselines, hyperparameters, etc.).

We consider four tasks:
\begin{enumerate*}[label=(\roman*)]
    \item cost minimization (MDP);
    \item maximum state entropy exploration (MSEE);
    \item imitation learning (IL); and
    \item multi-objective (MO) MDPs with different utility functions.
\end{enumerate*}
The definition of the objective function for each task is given in~\cref{sec:background:gumpds}. We consider two sets of environments: (i) illustrative environments that consist of low-dimensional GUMDPs; and (ii) high-dimensional GUMDPs consisting of grid-based environments showcasing different sources of stochasticity in order to encourage distinct types of policy behaviors. To our knowledge, we are the first work to propose an algorithm to solve risk-aware GUMDPs and, hence, there are no baselines that we can use to compare the performance of our method against. However, for the particular case of linear $f$ (MDP), we use an oracle baseline \textsc{ERM-BI} to validate the performance of our \textsc{ERM-MCTS} algorithm. \textsc{ERM-BI} computes the optimal risk-aware policy for the underlying occupancy MDP using a dynamic programming approach by exploiting the dynamic decomposition of the $\text{ERM}_\beta$ put forth by \cite{hau_2023}. All our plots are computed by aggregating the experimental results of, at least, 100 independent runs, and we refer to~\Cref{tab:appendix:illustrative-gumdps-hyperparams,tab:appendix:grid-gumdps-hyperparams} for the detailed list of our experimental hyperparameters across all environments.

\subsection{Illustrative environments} \label{sec:eval-ie}
We display in~\cref{fig:illustrative_envs} our illustrative environments. The MDP (\cref{fig:illustrative_envs}~(a)) consists of a four-state MDP where the agent needs to tradeoff at the initial state ($s_0$) between: (i) a risky action ($a_0$) that can lead to a low-cost state ($s_2$), but with some probability the agent ends in a high-cost state ($s_3$); and (ii) a safe action ($a_1$) that deterministically leads to a medium-cost state ($s_1$). The agent resets to the initial state with $10\%$ probability while not in the initial state. For the IL task (\cref{fig:illustrative_envs}~(a)), the agent aims to imitate the empirical occupancy induced by the trajectory of an agent that selected twice the risky action ($a_0$) and then selected the safe action ($a_1$) for all the remainder timesteps. The trajectory to imitate had a rather ``lucky'' outcome, as it never ended up in the absorbing state ($s_3$). Hence, the agent needs to trade off between imitating the behavior policy in states ($s_0$, $s_1$, $s_2$) and risking being absorbed into $s_3$, or only imitating the behavior policy in the ``less risky'' states ($s_0$, $s_1$). The MSEE task (\cref{fig:illustrative_envs}~(b)) is similar to the motivating example from \cref{sec:intro}, where the agent aims to explore an environment as uniformly as possible, but there is a chance that the agent transitions to an absorbing state ($s_4$). Finally, the MO-MDP (\cref{fig:illustrative_envs}~(c)) is inspired by the FishWood environment \citep{roijers_2020}, where the agent needs to tradeoff between two cost functions. Each cost function penalizes different behaviors. We consider three utility functions \citep{hayes2022montecarlotreesearch}: (i) a weighted combination of the discounted cumulative costs; (ii) the maximum of the discounted cumulative costs; and (iii) the minimum of the discounted cumulative costs.

\begin{figure*}[t]
    \centering
    \begin{minipage}{0.26\textwidth}
        \begin{subfigure}[b]{\textwidth}
            \centering
            \includegraphics[width=0.85\columnwidth]{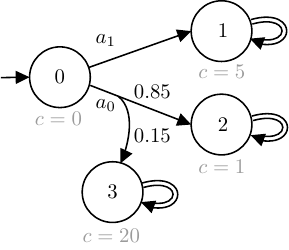}
            \caption{MDP and IL (10\% reset)}
        \end{subfigure}
    \end{minipage}
    \hfill
    \begin{minipage}{0.45\textwidth}
        \begin{subfigure}[b]{\textwidth}
            \centering
            \includegraphics[width=0.99\columnwidth]{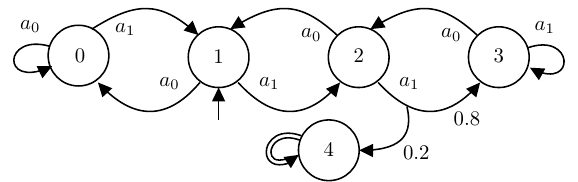}
            \caption{MSEE}
        \end{subfigure}
    \end{minipage}
    \hfill
    \begin{minipage}{0.27\textwidth}
        \begin{subfigure}[b]{\textwidth}
            \centering
            \includegraphics[width=0.99\columnwidth]{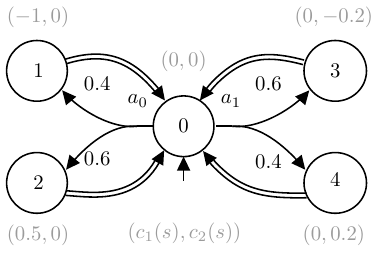}
            \caption{MO}
        \end{subfigure}
    \end{minipage}
    \caption{Illustrative environments.}
    \label{fig:illustrative_envs}
\end{figure*}

\begin{figure}[t]
    \begin{center}
        \includegraphics[width=0.85\linewidth]{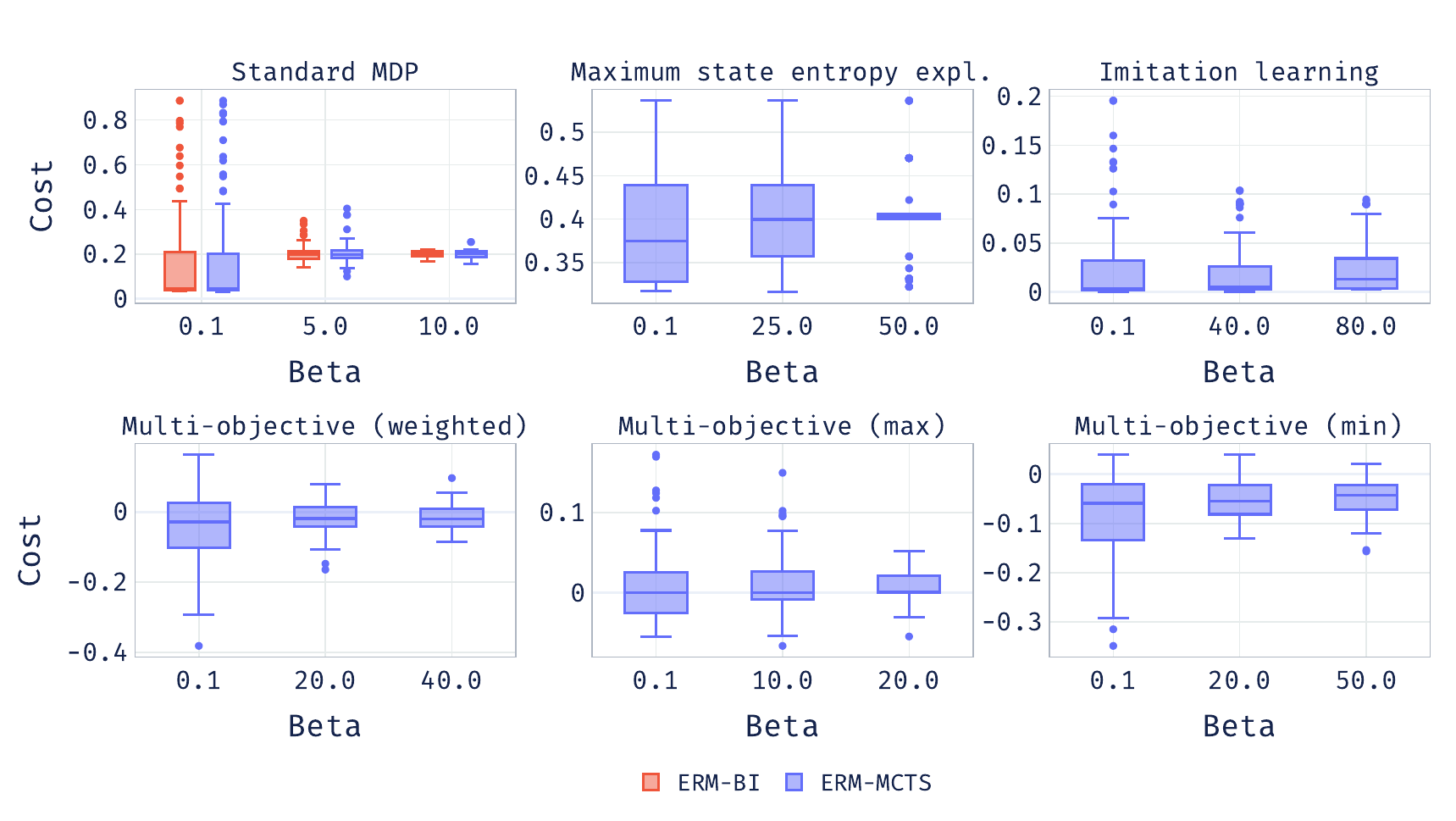}
    \end{center}
    \caption{Box plots of the costs obtained under the illustrative environments. Lower is better.}
    \label{fig:illustrative_envs_results}
\end{figure}

In ~\cref{fig:illustrative_envs_results}, we display the box plots computed for the empirical distributions of costs obtained by \textsc{ERM-MCTS} under different tasks and $\beta$ values. To validate our algorithm, we include a cost-minimization task corresponding to a standard MDP. As seen, for the MDP, the box plot obtained by \textsc{ERM-MCTS} closely matches that obtained by the \textsc{ERM-BI} oracle baseline across the tested $\beta$ values, providing empirical validation of \textsc{ERM-MCTS}. Furthermore, across all tasks, \textsc{ERM-MCTS} successfully trades off risk-neutral and risk-averse behavior as a function of $\beta$. As $\beta$ increases and \textsc{ERM-MCTS} progressively focuses on optimizing worst-case outcomes, upper whiskers and outliers (i.e., the worst outcomes) of the box plots decrease. This indicates that worst-case outcomes become less probable and that \textsc{ERM-MCTS} effectively computes increasingly risk-averse policies. Naturally, optimizing for worst-case outcomes typically degrades best-case or expected outcomes, as reflected by the increase in median values and in the lower whiskers and outliers as $\beta$ increases.

\subsection{Grid environments} \label{sec:eval-ge}
\subsubsection{Maximum State Entropy Exploration} \label{sec:grid_msee}
The MSEE task integrates a \(10\times10\) grid that comprises two different types of squares:
\begin{enumerate*}[label=(\roman*)]
    \item normal terrain; and
    \item difficult terrain.
\end{enumerate*}
The agent can move in all four directions, ending up in the corresponding adjacent square (if moving out-of-bounds, the agent remains in the same square). When the agent is inside a difficult terrain square, there is a \(p_\text{stuck}\) probability that the agent gets stuck, which translates in canceling the selected (moving) action and thus remaining trapped in the current square. Upon getting stuck, there is an additional \(p_\text{unstuck}\) probability that untraps the agent allowing it to move freely again. As expected, inside normal terrain, the agent's selected moving action is always guaranteed to occur. Due to the environment's stochastic nature, the agent needs to consider if there are benefits in exploring certain difficult terrain squares at the expense of getting stuck, which can have a negative effect on the discounted costs received from that point onward. 

In \Cref{fig:grid_envs_msee}, we report the box plots of the empirical costs obtained by \textsc{ERM-MCTS} with \(\beta\in\{0.001,1,1000\}\). When \(\beta\in\{0.001,1\}\), \textsc{ERM-MCTS} outputs similar risk-seeking policies, as both display identical Inner Fence Interval (IFI) ranges, inside \([0.1, 0.2]\). Additionally, the expectation value is lower when \(\beta=0.001\), having a value of \(0.167\), whereas the expectation for \(\beta=1\) is \(0.177\). Furthermore, both policies exhibit a wide range of outliers on the upper-half of the distribution, implying that in some runs the agent got trapped for long periods of time when covering difficult terrain squares, further alluding to the potential disadvantages of having a risk-seeking behavior. Additionally, the worst outlier is higher for the policy with \(\beta=0.001\), which is expected since this policy is effectively closer to the risk-seeking (or risk-neutral) boundary in the \emph{risk-awareness} spectrum. When comparing the policy obtained when \(\beta=1000\) to the the previous two (\(\beta\in\{0.001,1\}\)), the range of the objective distribution becomes more concentrated as there are fewer outliers. This shows the effectiveness of the risk-averse behavior in avoiding the worst outcomes, as the highest cost is only \(0.48\) when setting \(\beta=1000\), vs. \(0.93\) using \(\beta=0.001\). The drawback of this policy manifests by inducing an average case that performs worse than policies that are more risk-seeking, where the average of the objective distributions are \(0.30\) and \(0.167\), for the policies using \(\beta=1000\) and \(\beta=0.001\), respectively. Finally, we further validate our analysis with \Cref{fig:grid_envs_msee_heatmap}, which exposes the average number of timesteps spent in each square over all runs. We observe that when \(\beta=1000\), the policy navigates towards the lower-right corner to avoid difficult terrain squares altogether. As \(\beta\) decreases, the heatmap becomes more uniformly distributed showcasing that the policy tries to visit all squares (normal and difficult terrain) more often.

\begin{figure}[t]
\centering
\begin{subfigure}{.32\textwidth}
  \centering
  \includegraphics[width=\linewidth]{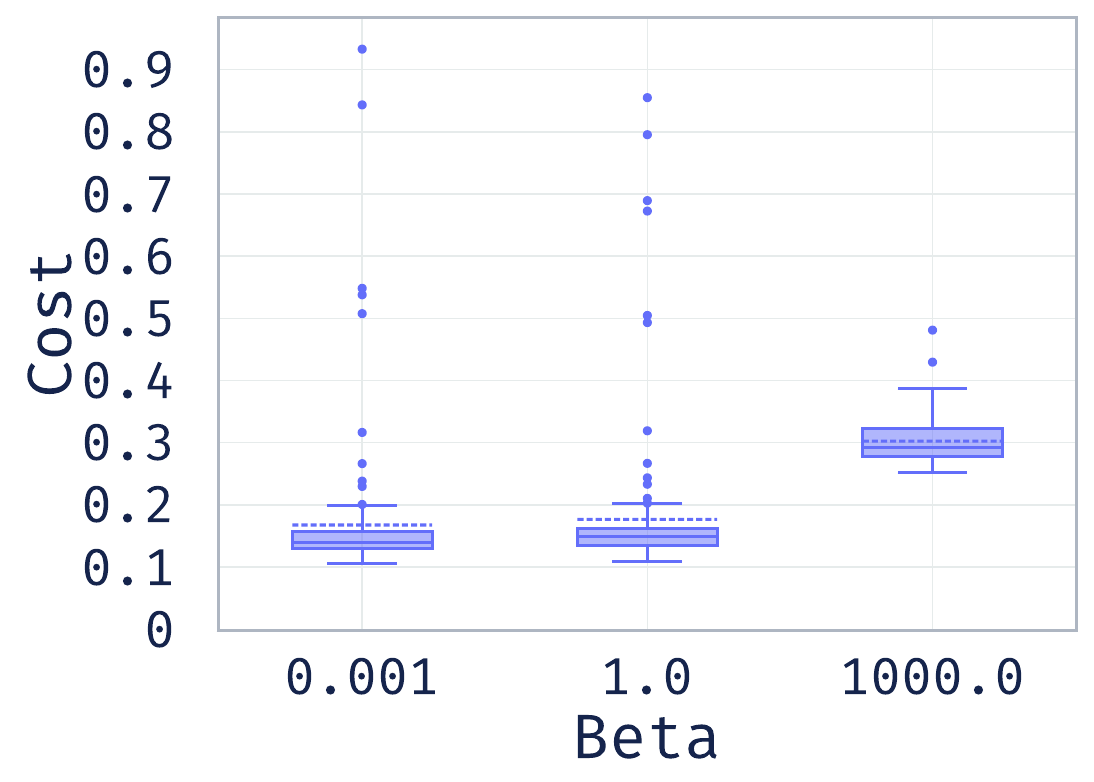}
  \caption{MSEE}
  \label{fig:grid_envs_msee}
\end{subfigure}%
\hfill
\begin{subfigure}{.32\textwidth}
  \centering
  \includegraphics[width=\linewidth]{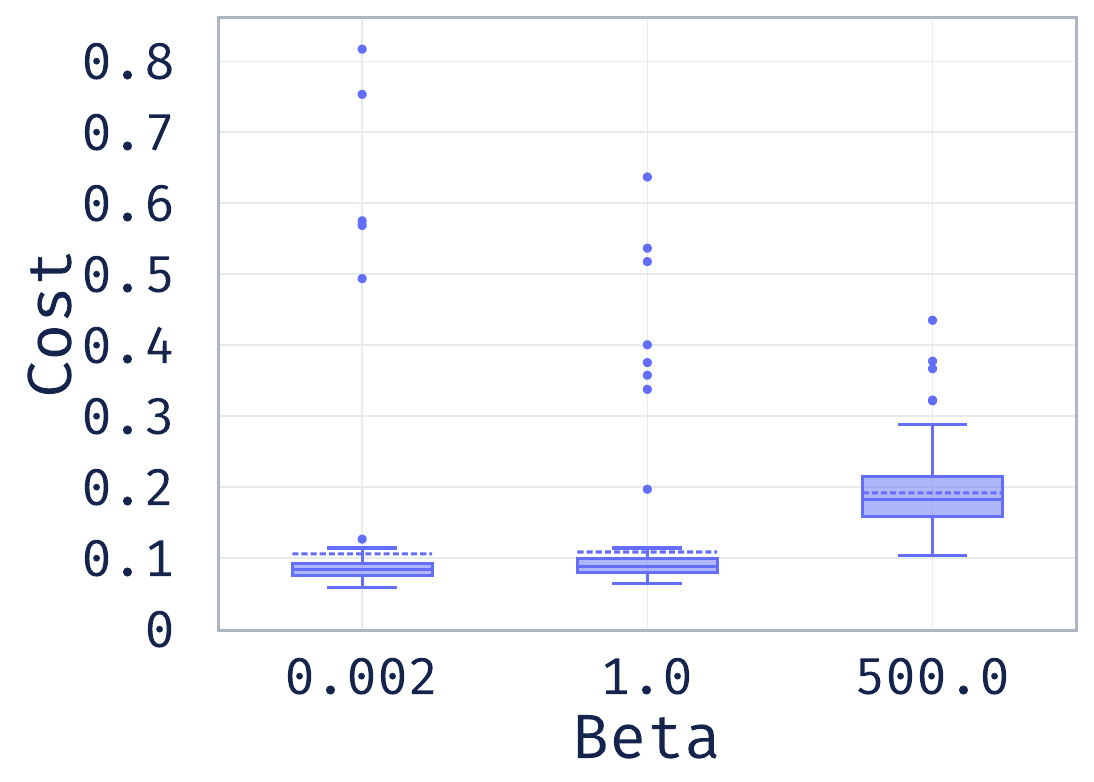}
  \caption{IL}
  \label{fig:grid_envs_il}
\end{subfigure}%
\hfill
\begin{subfigure}{.32\textwidth}
  \centering
  \includegraphics[width=\linewidth]{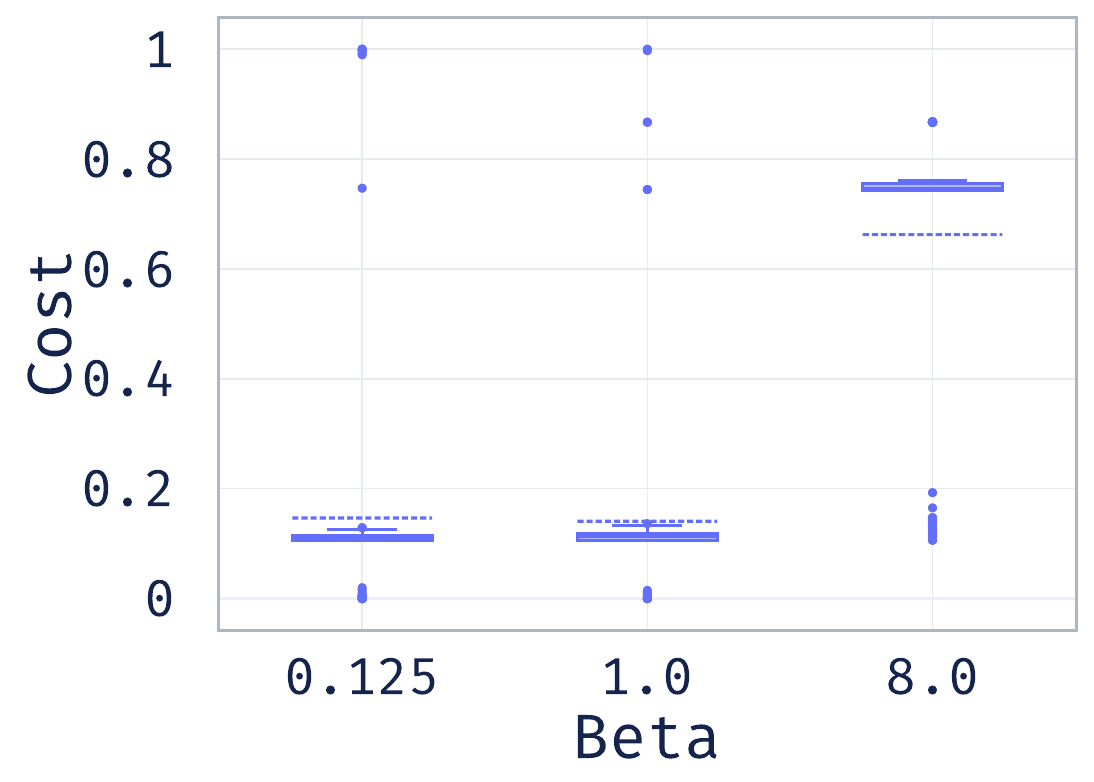}
  \caption{MO}
  \label{fig:grid_envs_mo}
\end{subfigure}%
\caption{Grid environments: box plots obtained, for different tasks, by deploying \textsc{ERM-MCTS} with different \(\beta\) values. Box plots computed for 128 independent runs. Lower is better.}
\end{figure}

\begin{figure}[t]
    \begin{center}
        \includegraphics[width=0.85\linewidth,trim={0 26ex 0 31ex},clip]{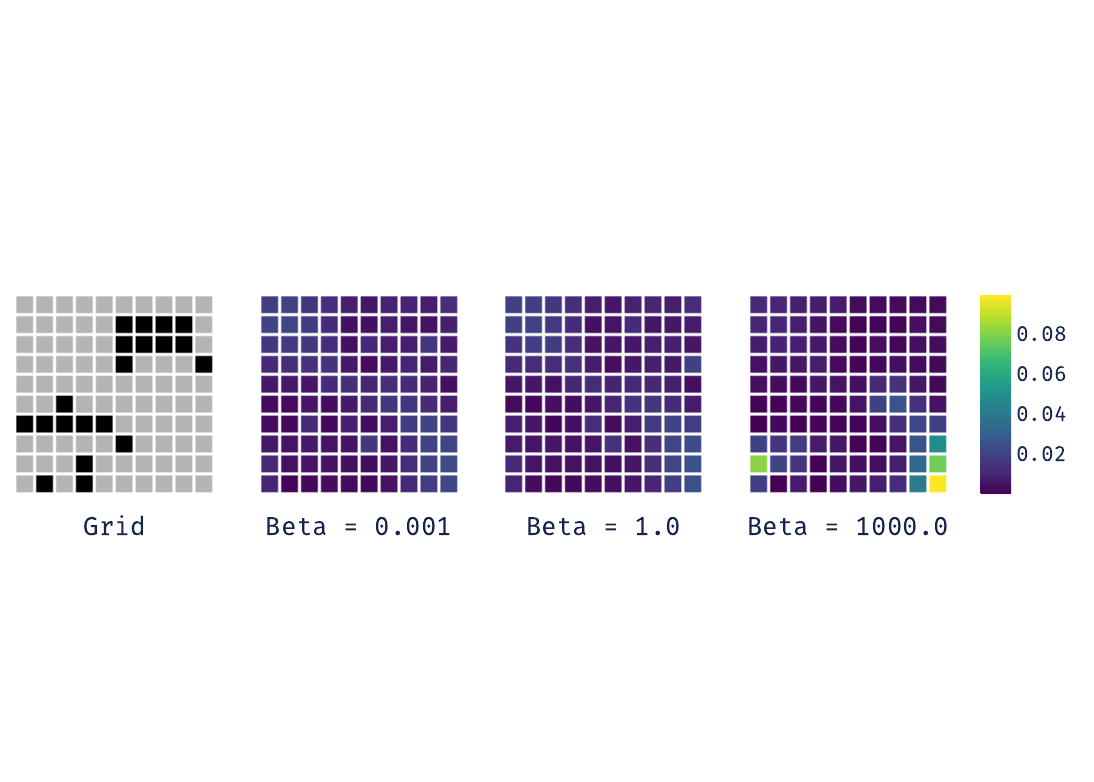}
    \end{center}
    \caption{MSEE (grid environment): leftmost grid illustrates the environment, where normal and difficult terrain are depicted in gray and black squares, respectively; the following grids showcase the average time spent in each square averaged over all (128) runs, for each \(\beta\in\{0.001,1,1000\}\).}
    \label{fig:grid_envs_msee_heatmap}
\end{figure}

\subsubsection{Imitation Learning} \label{sec:grid_il}
We reuse the MSEE environment for our IL task. Additionally, we carefully design a behavior policy that changes from risk-seeking and risk-aware behaviors at well-defined intervals. In high-level terms, the behavior policy starts by exploring a zone containing both types of squares, followed by a period where only normal squares are visited. This procedure repeats a second time to make sure all grid squares are visited once. Furthermore, the behavior policy will not be subjected to the environment's stochasticity, i.e., it will never get trapped when visiting difficult terrain. In this manner, we guarantee that achieving a complete match will be extremely unlikely, further leading the \textsc{ERM-MCTS} policy to consider when and what risks to take.

To analyze the IL task, we compare \textsc{ERM-MCTS} with three different beta values, \(\beta\in\{0.02,1,500\}\). As shown in~\Cref{fig:grid_envs_il}, we clearly observe the same pattern reported in the MSEE task (\Cref{sec:grid_msee}): as \(\beta\) increases, the maximum cost obtained decreases while the expectation and median increase. Here, the expectation obtains the values \(0.106\), \(0.109\), and \(0.192\), while the median acquires \(0.08\), \(0.09\), and \(0.18\) for the policies with \(\beta=0.02\), \(\beta=1\), and \(\beta=500\), respectively. Conversely, and for the same policies, the uppermost outliers are \(0.81\), \(0.65\), and \(0.44\), respectively. To get a closer look at how behaviors differ as \(\beta\) changes, we plot, in~\Cref{fig:grid_envs_il_trap}, the average number of runs visiting difficult terrain squares in each timestep. Interestingly, the policy with the highest \(\beta\) emphasizes matching the occupancy of normal squares during the first 40 time steps, as there is a negligible number of runs stepping inside difficult terrain. For the remaining time, the policy focuses more on equalizing the occupancy at difficult terrain squares, as outlined by the increased number of runs exploring such squares. We argue that such behavior is due to discounting, since it is more forgiving to get trapped closer to the end of the episode, considering that it has less impact on the discounted return. Contrastingly, the other policies (\(\beta\in\{0.02,1\}\)) start visiting difficult terrain squares from the beginning of the episode. Consequently, the likelihood of getting trapped sooner and for longer periods of time increases, as hinted in~\Cref{fig:grid_envs_il} by the wide range of outliers on the upper half of the cost distribution.

\subsubsection{Multi-Objective} \label{sec:grid-mo}
For the MO task, we extend the resource-gathering environment~\citep{mo-rg}, aiming at introducing challenging probabilistic dynamics that allow the modulation of policies demonstrating behaviors with different degrees of risk-awareness. In the modified environment, the agent's objective consists in retrieving two different resources (\(R_1\) and \(R_2\)) back to its starting location while avoiding enemies present at specific squares. Additionally, we set a slippery dynamic, where with probability \(p_\text{slip}\) the agent will move perpendicularly to the selected direction. All agent objectives are modulated using the following reward functions:
\begin{enumerate*}[label=(\roman*)]
    \item agent receives a cost when failing to deliver \(R_1\) to the home location (the agent's initial position);
    \item agent receives a cost when failing to delivering \(R_2\) to the home location;
    \item agent receives a cost when getting defeated by an enemy, which happens, with probability \(p_\text{defeat}\), when visiting a square containing an enemy.
\end{enumerate*}
Finally, we combine the three independent reward functions using a non-linear mapping that prioritizes collecting \(R_1\) over the other ones. We refer to \Cref{app:mo-grid} for a complete description of the dynamics, cost functions, and hyperparameters used.


We simulate \textsc{ERM-MCTS} with \(\beta\in\{0.125,1,8\}\). As shown in \Cref{fig:grid_envs_mo}, the cost distributions obtained by the policies where \(\beta\in\{0.125,1\}\) are highly identical. We again observe the risk-seeking behavior of these policies, since, on average, the costs obtained remain concentrated near the lower bound. For instance, the expectation is \(0.16\) and \(0.18\) for \(\beta=0.125\) and \(\beta=1\), respectively. The IFI is \([0.10,0.14]\) in both cases. Furthermore, since during the optimization of the ERM's objective, the expectation outweighs the upper tail of the distribution, eventually some runs conclude at the worst possible outcome, in which the agent gets defeated by an enemy. Contrastingly, the policy generated for \(\beta=8\) clearly eliminates this possible outcome. In this case, the agent never gets the cost of \(1\) by inducing a risk-aware behavior that most of the time prefers not to collect \(R_1\) due to its proximity to enemies and, as such, the likelihood of defeat. Therefore, the median centers around \(0.87\), indicating the agent never retrieved any resources to its initial position. Finally, in \Cref{fig:grid_envs_mo_outcome}, we show in more detail the frequency of outcomes that each policy achieves. As hinted previously, when \(\beta\in\{0.125,1\}\) we observe a small fraction of runs where the agent gets defeated. On the other hand, the risk-averse policy (\(\beta=8\)) disallows this outcome entirely by, most of the time (82\%), either avoiding gathering any resources or just retrieving \(R_2\).

\begin{figure}[t]
\centering
\begin{minipage}{.47\textwidth}
  \centering
  \includegraphics[width=\linewidth]{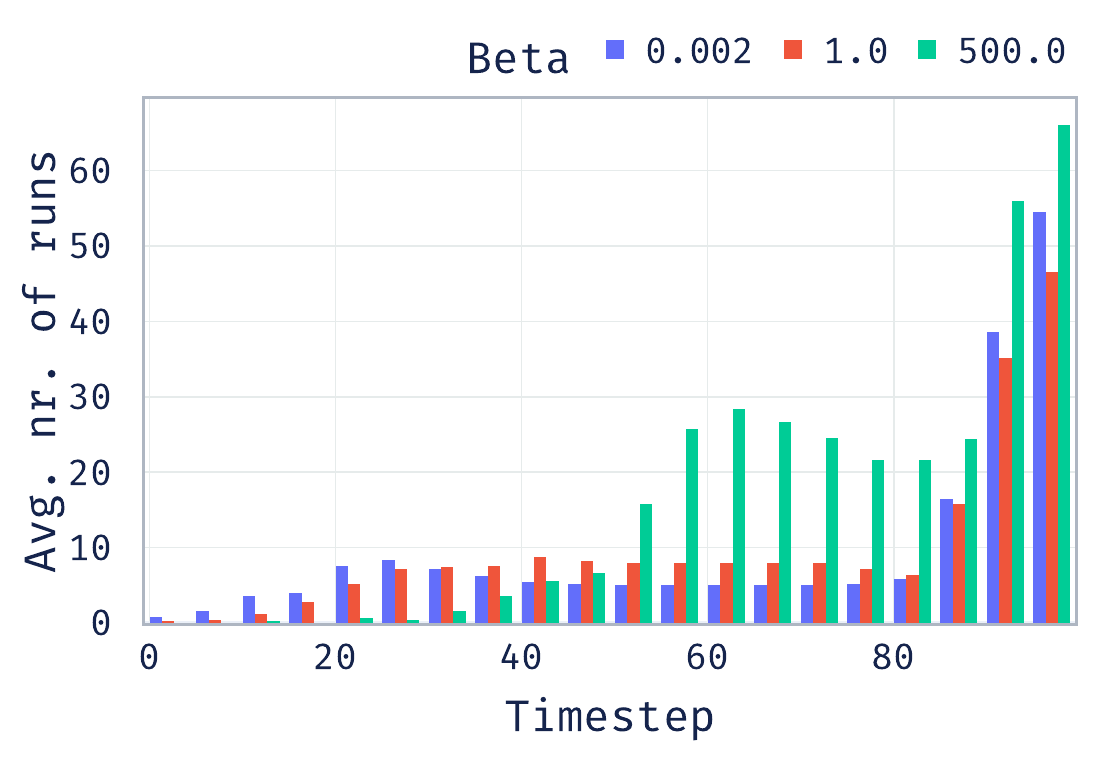}
  \captionof{figure}{IL (grid environment): Average number of runs visiting difficult terrain squares in each timestep for different \(\beta\) values.}
  \label{fig:grid_envs_il_trap}
\end{minipage}%
\hfill
\begin{minipage}{.47\textwidth}
  \centering
  \includegraphics[width=\linewidth]{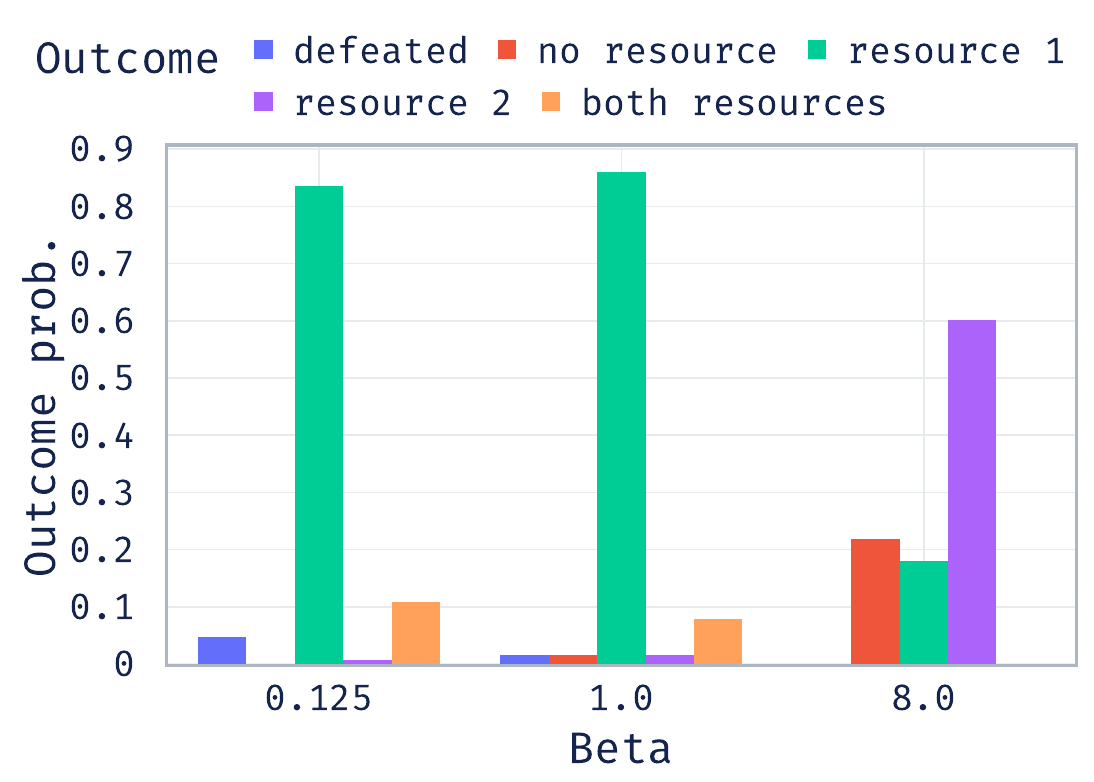}
  \captionof{figure}{MO (grid environment): Probability of each outcome for different \(\beta\) values.}
  \label{fig:grid_envs_mo_outcome}
\end{minipage}
\end{figure}

\section{Conclusion}
We motivate, propose and formalize risk-aware GUMDPs, which allow to take advantage of the flexibility of the GUMDPs framework with respect to objective specification, while trading off expected performance and risk-aversion. To solve risk-aware GUMDPs, we first explore a connection between risk-aware GUMDPs and solving a particular risk-aware MDP, named occupancy MDP, in which the agent keeps track of the empirical frequency of visitation of state-action pairs up to the current timestep. We propose a provably correct online planning approach based on an MCTS algorithm to solve the risk-aware occupancy MDP, effectively solving the original risk-aware GUMDP problem up to any desired accuracy. We provide a set of experimental results under a set of diverse tasks showing that our approach successfully trades off risk-neutral and risk-averse behavior.

Future work could investigate whether a similar approach to ours can be used to solve GUMDPs with CVaR objectives \citep{rockafellar2000optimization,NIPS2015_64223ccf}. Other interesting direction is to investigate whether our approach can be extended to deal with very large/inherently continuous state spaces, e.g., by borrowing ideas from successor features \citep{barreto_2017,borsa_2018}.

\subsubsection*{Acknowledgments}
\label{sec:ack}
This work was supported by Portuguese national funds through the Portuguese Fundação para a Ciência e a Tecnologia (FCT) under projects UID/50021/2025 and UID/PRR/50021/2025 (INESC-ID multi-annual funding), as well as AI-PackBot (project number 14935, LISBOA2030-FEDER-00854700). Pedro P. Santos acknowledges the FCT PhD grant 2021.04684.BD and Fábio Vital acknowledges the FCT PhD grant 2022.14163.BD. Alberto Sardinha acknowledges the CNPq Research Productivity Fellowship (PQ), with reference 312699/2025-5. The authors thank the lab managers at GAIPS for the support provided when running the computational experiments of this work. The authors also thank Jacopo Silvestrin for discussions on earlier versions of this work and Zita Marinho for feedback on the manuscript.


\bibliography{main}

@article{rockafellar2000optimization,
  title={Optimization of conditional value-at-risk},
  author={Rockafellar, R Tyrrell and Uryasev, Stanislav and others},
  journal={Journal of risk},
  volume={2},
  pages={21--42},
  year={2000},
  publisher={Citeseer}
}

@book{puterman2014markov,
  author = {Puterman, Martin L},
  publisher = {John Wiley \& Sons},
  title = {Markov decision processes: discrete stochastic dynamic programming},
  year = 2014
}

@book{lattimore_2020, title={Bandit Algorithms}, DOI={10.1017/9781108571401}, publisher={Cambridge University Press}, author={Lattimore, Tor and Szepesvári, Csaba}, year={2020}}

@article{zahavy_2021,
  author       = {Tom Zahavy and
                  Brendan O'Donoghue and
                  Guillaume Desjardins and
                  Satinder Singh},
  title        = {Reward is enough for convex MDPs},
  journal      = {CoRR},
  volume       = {abs/2106.00661},
  year         = {2021},
}

@InProceedings{hazan_2019,
  title = 	 {Provably Efficient Maximum Entropy Exploration},
  author =       {Hazan, Elad and Kakade, Sham and Singh, Karan and Van Soest, Abby},
  booktitle = 	 {Proceedings of the 36th International Conference on Machine Learning},
  pages = 	 {2681--2691},
  year = 	 {2019},
  volume = 	 {97},
  series = 	 {Proceedings of Machine Learning Research},
}

@article{mutti_2023,
  author  = {Mirco Mutti and Riccardo De Santi and Piersilvio De Bartolomeis and Marcello Restelli},
  title   = {Convex Reinforcement Learning in Finite Trials},
  journal = {Journal of Machine Learning Research},
  year    = {2023},
  volume  = {24},
  number  = {250},
  pages   = {1--42},
}

@article{dvoretzky_1952,
 author = {A. Dvoretzky and J. Kiefer and J. Wolfowitz},
 journal = {Econometrica},
 number = {3},
 pages = {450--466},
 publisher = {[Wiley, Econometric Society]},
 title = {The Inventory Problem: II. Case of Unknown Distributions of Demand},
 volume = {20},
 year = {1952}
}

@book{chow1971great,
  title={Great Expectations: The Theory of Optimal Stopping},
  author={Chow, Y.S. and Robbins, H. and Siegmund, D.},
  isbn={9780395053140},
  year={1971},
}

@article{stidham_1978,
 author = {Shaler Stidham},
 journal = {Management Science},
 number = {15},
 pages = {1598--1610},
 title = {Socially and Individually Optimal Control of Arrivals to a GI/M/1 Queue},
 volume = {24},
 year = {1978}
}

@book{sutton_1998,
author= {Richard S. Sutton and Andrew G. Barto},
  edition = {Second},
  publisher = {The MIT Press},
  title = {Reinforcement Learning: An Introduction},
  year = {2018 }
}

@article{mnih_2015,
  author    = {Mnih, Volodymyr and
               Kavukcuoglu, Koray and
               Silver, David and
               Graves, Alex and
               Antonoglou, Ioannis and
               Wierstra, Daan and
               Riedmiller, Martin},
  title     = {Playing Atari with Deep Reinforcement Learning},
  journal   = {Nature},
    year={2015},
    volume={518},
    number={7540},
    pages={529-533},
}

@article{silver_2017,
author={Silver, David
and Schrittwieser, Julian
and Simonyan, Karen
and Antonoglou, Ioannis
and Huang, Aja
and Guez, Arthur
and Hubert, Thomas
and Baker, Lucas
and Lai, Matthew
and Bolton, Adrian
and Chen, Yutian
and Lillicrap, Timothy
and Hui, Fan
and Sifre, Laurent
and van den Driessche, George
and Graepel, Thore
and Hassabis, Demis},
title={Mastering the game of Go without human knowledge},
journal={Nature},
year={2017},
volume={550},
number={7676},
pages={354-359},
}

@article{lillicrap_2016,
  title={Continuous control with deep reinforcement learning},
  author={T. Lillicrap and J. Hunt and A. Pritzel and N. Heess and T. Erez and Y. Tassa and D. Silver and Daan Wierstra},
  journal={CoRR},
  year={2016},
  volume={abs/1509.02971}
}

@misc{abel2022expressivity,
      title={On the Expressivity of Markov Reward}, 
      author={David Abel and Will Dabney and Anna Harutyunyan and Mark K. Ho and Michael L. Littman and Doina Precup and Satinder Singh},
      year={2022},
      eprint={2111.00876},
      archivePrefix={arXiv},
      primaryClass={cs.LG}
}

@article{hussein_2017,
author = {Hussein, Ahmed and Gaber, Mohamed Medhat and Elyan, Eyad and Jayne, Chrisina},
title = {Imitation Learning: A Survey of Learning Methods},
year = {2017},
volume = {50},
number = {2},
journal = {ACM Comput. Surv.},
month = {apr},
articleno = {21},
numpages = {35},
}

@article{Osa_2018,
   title={An Algorithmic Perspective on Imitation
                  Learning},
   volume={7},
   ISSN={1935-8261},
   number={1–2},
   journal={Foundations and Trends in Robotics},
   author={Osa, Takayuki and Pajarinen, Joni and Neumann, Gerhard and Bagnell, J. Andrew and Abbeel, Pieter and Peters, Jan},
   year={2018},
   pages={1–179} }

@article{garcia_2015,
  author  = {Javier Garc{{\'i}}a and Fern and o Fern{{\'a}}ndez},
  title   = {A Comprehensive Survey on Safe Reinforcement Learning},
  journal = {Journal of Machine Learning Research},
  year    = {2015},
  volume  = {16},
  number  = {42},
  pages   = {1437--1480},
}

@misc{eysenbach2018diversity,
      title={Diversity is All You Need: Learning Skills without a Reward Function}, 
      author={Benjamin Eysenbach and Abhishek Gupta and Julian Ibarz and Sergey Levine},
      year={2018},
      eprint={1802.06070},
      archivePrefix={arXiv},
      primaryClass={cs.AI}
}

@misc{achiam2018variational,
      title={Variational Option Discovery Algorithms}, 
      author={Joshua Achiam and Harrison Edwards and Dario Amodei and Pieter Abbeel},
      year={2018},
      eprint={1807.10299},
      archivePrefix={arXiv},
      primaryClass={cs.AI}
}

@book{altman_1999,
  author = {Altman, E.},
  publisher = {Chapman and Hall},
  title = {Constrained Markov Decision Processes},
  year = 1999
}

@article{efroni_2020,
  author       = {Yonathan Efroni and
                  Shie Mannor and
                  Matteo Pirotta},
  title        = {Exploration-Exploitation in Constrained MDPs},
  journal      = {CoRR},
  volume       = {abs/2003.02189},
  year         = {2020},
}

@misc{geist2022concave,
      title={Concave Utility Reinforcement Learning: the Mean-Field Game Viewpoint}, 
      author={Matthieu Geist and Julien Pérolat and Mathieu Laurière and Romuald Elie and Sarah Perrin and Olivier Bachem and Rémi Munos and Olivier Pietquin},
      year={2022},
      eprint={2106.03787},
      archivePrefix={arXiv},
      primaryClass={cs.LG}
}

@misc{santos_2024,
      title={The Number of Trials Matters in Infinite-Horizon General-Utility Markov Decision Processes}, 
      author={Pedro P. Santos and Alberto Sardinha and Francisco S. Melo},
      year={2024},
      eprint={2409.15128},
      archivePrefix={arXiv},
      primaryClass={cs.LG},
}

@misc{santos_2025,
      title={Solving General-Utility Markov Decision Processes in the Single-Trial Regime with Online Planning}, 
      author={Pedro P. Santos and Alberto Sardinha and Francisco S. Melo},
      year={2025},
      eprint={2505.15782},
      archivePrefix={arXiv},
      primaryClass={cs.LG},
}

@article{artzner_1999,
author = {Artzner, Philippe and Delbaen, Freddy and Jean-Marc, Eber and Heath, David},
year = {1999},
month = {07},
pages = {203 - 228},
title = {Coherent Measures of Risk},
volume = {9},
journal = {Mathematical Finance},
}

@inproceedings{abbeel_2004,
author = {Abbeel, Pieter and Ng, Andrew Y.},
title = {Apprenticeship learning via inverse reinforcement learning},
year = {2004},
booktitle = {Proceedings of the Twenty-First International Conference on Machine Learning},
pages = {1},
series = {ICML '04}
}

@inproceedings{NIPS2015_64223ccf,
 author = {Chow, Yinlam and Tamar, Aviv and Mannor, Shie and Pavone, Marco},
 booktitle = {Advances in Neural Information Processing Systems},
 pages = {},
 title = {Risk-Sensitive and Robust Decision-Making: a CVaR Optimization Approach},
 volume = {28},
 year = {2015}
}

@Article{howard_matheson_1972,
journal={Management Science},
author={Ronald A. Howard and James E. Matheson},
title={Risk-Sensitive Markov Decision Processes},
year={1972},
month={March},
pages={356-369},
volume={18},
number={7},
}

@InProceedings{hau_2023,
  title = 	 {Entropic Risk Optimization in Discounted {MDPs}},
  author =       {Lin Hau, Jia and Petrik, Marek and Ghavamzadeh, Mohammad},
  booktitle = 	 {Proceedings of The 26th International Conference on Artificial Intelligence and Statistics},
  pages = 	 {47--76},
  year = 	 {2023},
  volume = 	 {206},
  month = 	 {25--27 Apr},
  publisher =    {PMLR},
}

@book{follmer_2016,
  title={Stochastic Finance: An Introduction in Discrete Time},
  author={F{\"o}llmer, H. and Schied, A.},
  isbn={9783110463453},
  series={De Gruyter Textbook},
  year={2016},
}

@article{borkar_2002,
 author = {V. S. Borkar and S. P. Meyn},
 journal = {Mathematics of Operations Research},
 number = {1},
 pages = {192--209},
 title = {Risk-Sensitive Optimal Control for Markov Decision Processes with Monotone Cost},
 volume = {27},
 year = {2002}
}

@inproceedings{pagnoncelli_2022,
  TITLE = {{Multistage stochastic programs with the entropic risk measure}},
  AUTHOR = {Pagnoncelli, Bernardo and Dowson, Oscar and Morton, David},
  YEAR = {2022},
  MONTH = Feb,
}

@misc{marthe2025efficientrisksensitiveplanningentropic,
      title={Efficient Risk-sensitive Planning via Entropic Risk Measures}, 
      author={Alexandre Marthe and Samuel Bounan and Aurélien Garivier and Claire Vernade},
      year={2025},
      eprint={2502.20423},
      archivePrefix={arXiv},
}

@misc{mortensen2025entropicriskoptimizationdiscounted,
      title={Entropic Risk Optimization in Discounted MDPs: Sample Complexity Bounds with a Generative Model}, 
      author={Oliver Mortensen and Mohammad Sadegh Talebi},
      year={2025},
      eprint={2506.00286},
      archivePrefix={arXiv},
      primaryClass={cs.LG},
}

@article{bental_1986,
 author = {Aharon Ben-Tal and Marc Teboulle},
 journal = {Management Science},
 number = {11},
 pages = {1445--1466},
 title = {Expected Utility, Penalty Functions, and Duality in Stochastic Nonlinear Programming},
 volume = {32},
 year = {1986}
}

@misc{santos2026entropicriskawaremontecarlo,
      title={Entropic Risk-Aware Monte Carlo Tree Search}, 
      author={Pedro P. Santos and Jacopo Silvestrin and Alberto Sardinha and Francisco S. Melo},
      year={2026},
      eprint={2601.17667},
      archivePrefix={arXiv},
      primaryClass={cs.LG},
}

@article{radulescu_2019,
  author       = {Roxana Radulescu and
                  Patrick Mannion and
                  Diederik M. Roijers and
                  Ann Now{\'{e}}},
  title        = {Multi-Objective Multi-Agent Decision Making: {A} Utility-based Analysis
                  and Survey},
  journal      = {CoRR},
  volume       = {abs/1909.02964},
  year         = {2019},
}

@misc{hayes2022montecarlotreesearch,
      title={Monte Carlo Tree Search Algorithms for Risk-Aware and Multi-Objective Reinforcement Learning}, 
      author={Conor F. Hayes and Mathieu Reymond and Diederik M. Roijers and Enda Howley and Patrick Mannion},
      year={2022},
      eprint={2211.13032},
      archivePrefix={arXiv},
}

@conference{roijers_2020,
title = "Multi-objective reinforcement learning for the expected utility of the return",
author = "Roijers, \{Diederik M.\} and Denis Steckelmacher and Ann Now{\'e}",
year = "2020",
note = "2018 Adaptive Learning Agents, ALA 2018 - Co-located Workshop at the Federated AI Meeting, FAIM 2018"
}

@inproceedings{barreto_2017,
 author = {Barreto, Andre and Dabney, Will and Munos, Remi and Hunt, Jonathan J and Schaul, Tom and van Hasselt, Hado P and Silver, David},
 booktitle = {Advances in Neural Information Processing Systems},
 pages = {},
 title = {Successor Features for Transfer in Reinforcement Learning},
 volume = {30},
 year = {2017}
}

@article{borsa_2018,
  author       = {Diana Borsa and
                  Andr{\'{e}} Barreto and
                  John Quan and
                  Daniel J. Mankowitz and
                  R{\'{e}}mi Munos and
                  Hado van Hasselt and
                  David Silver and
                  Tom Schaul},
  title        = {Universal Successor Features Approximators},
  journal      = {CoRR},
  volume       = {abs/1812.07626},
  year         = {2018},
}

@inproceedings{mo-rg,
author = {Barrett, Leon and Narayanan, Srini},
title = {Learning all optimal policies with multiple criteria},
year = {2008},
booktitle = {Proceedings of the 25th International Conference on Machine Learning},
pages = {41–47},
numpages = {7},
series = {ICML '08}
}

@book{bdr2023,
    title={Distributional Reinforcement Learning},
    author={Marc G. Bellemare and Will Dabney and Mark Rowland},
    publisher={MIT Press},
    year={2023}
}

@article{bauerle_2011,
author={B{\"a}uerle, Nicole
and Ott, Jonathan},
title={Markov Decision Processes with Average-Value-at-Risk criteria},
journal={Mathematical Methods of Operations Research},
year={2011},
month={Dec},
volume={74},
number={3},
pages={361-379},
issn={1432-5217},
}

@misc{zhang_2020,
      title={Cautious Reinforcement Learning via Distributional Risk in the Dual Domain}, 
      author={Junyu Zhang and Amrit Singh Bedi and Mengdi Wang and Alec Koppel},
      year={2020},
      eprint={2002.12475},
      archivePrefix={arXiv},
      primaryClass={stat.ML},
}
\bibliographystyle{rlj}

\crefalias{section}{appendix}
\crefalias{subsection}{appendix}
\crefalias{subsubsection}{appendix}
\beginSupplementaryMaterials

\section{Extended related work discussion} \label{appendix:related_work}
Previous works studied GUMDPs with risk-neutral objectives \citep{zahavy_2021,geist2022concave}. \citet{zahavy_2021} reformulate GUMDPs as a two-player game involving a policy and a cost (negative reward) player, using Fenchel duality. \citet{geist2022concave} connect GUMDPs and mean-field games. Still in the risk-neutral setting, subsequent works identified a key implicit assumption underlying the GUMDPs framework: the performance of a given policy may depend on the number of trials/trajectories drawn to evaluate its performance \citep{mutti_2023,santos_2024}.  In fact, the previous works show that the standard formulation of GUMDPs implicitly assumes the performance of a given policy is evaluated under an infinite number of trials/trajectories, an assumption that may be violated under many interesting application domains. To address this gap, \citet{mutti_2023,santos_2025} introduce finite trials formulations for GUMDPs, allowing to find optimal risk-neutral policies with respect to a finite number of trials/trajectories. None of the works above study risk-aware settings, being focused on learning risk-neutral policies. As pointed out in the main text, some of the results in our paper build on top of results put forth by the aforementioned studies. As an example, our work relies on the occupancy MDP formulation, as described in Sec.~\ref{sec:occupancy_mdp}, which was introduced in \citep{santos_2025}.

In the context of risk-aware MDPs, previous works studied the optimization of risk-aware objectives such as the conditional value-at-risk \cite{NIPS2015_64223ccf} or the ERM \citep{hau_2023,marthe2025efficientrisksensitiveplanningentropic,mortensen2025entropicriskoptimizationdiscounted}. In the context of our work, we are focused on more general objective functions of the occupancies induced by the different policies, not only being focused on the case of linear objectives (MDPs). Nevertheless, some of the techniques we employ resemble those adopted in the context of risk-aware MDPs. As an example, it is common in the context of risk-aware MDPs to extend the state space to accommodate sufficient statistics since the optimal policy depends on the history only through a certain kind of ``sufficient statistic'' \citep{bauerle_2011}. This resembles the construction of our occupancy MDP; however, the states in the occupancy MDP do not correspond to sufficient statistics of the history but instead correspond directly to histories (Lemma~\ref{lemma:one_to_one_mapping_histories_states}). This is required since, under more general settings, optimal policies may depend arbitrarily on the history. Other works, such as \citet{hau_2023}, exploit discounting and consider finite-horizon objectives to compute approximately optimal policies for infinite-horizon objectives - this technique is common in the field, and we followed a similar approach in \ref{proposition:opt_gap_decomposition}.

Finally, in the context of MDPs, \citet{zhang_2020} propose a new definition of the risk of a policy (called caution) as a function of the occupancy induced by the policy. Then, the authors formulate a caution-sensitive policy
optimization problem by adding the caution risk as a penalty function to the dual objective of the linear programming formulation of MDPs. The formulation proposed by the authors can itself be seen as solving a risk-neutral GUMDP, where the objective function contains a term that captures the risk of a given policy. In our work, we are instead focused on solving risk-aware GUMDPs.

\section{Lipschitz constants} \label{appendix:A}

\begin{table}[h]
\centering
\caption{Lipschitz constants for common objective functions found in the GUMDPs literature. In $(\dagger)$ we assume $\vec{d}$ is lower bounded by $\epsilon$ satisfying $0 < \epsilon < e^{-2}$.}
\label{tab:objective_functions}
\begin{tabular}{c c c }
\toprule
\textbf{Task} & \textbf{Objective ($f(\vec{d})$)} & \textbf{Lipschitz constant ($L$)} \\ \cmidrule(ll){1-1} \cmidrule(ll){2-2} \cmidrule(ll){3-3}
MDPs/RL & $\vec{d}^\top \vec{c}, \quad \vec{c} \in \mathbb{R}^{|\S||\A|}$ & $\max_{s,a} |c(s,a)|$ \\ \cmidrule(ll){1-1} \cmidrule(ll){2-2} \cmidrule(ll){3-3}
Max. state entropy expl. & $\vec{d}^\top\log(\vec{d})$ & $|\log(\epsilon) + 1 |\; (\dagger)$ \\ \cmidrule(ll){1-1} \cmidrule(ll){2-2} \cmidrule(ll){3-3}
Imitation learning & $\| \vec{d} - \vec{d}_{\pi_b} \|_2^2, \quad \vec{d}_{\pi_b} \in \Delta(\S \times \A)$ & $4$
\\ \cmidrule(ll){1-1} \cmidrule(ll){2-2} \cmidrule(ll){3-3}
Multi-objective MDPs & $n_1 \vec{d}^\top \vec{c}_1 + \ldots + n_k\vec{d}^\top \vec{c}_k$ & $\max_{s,a} |\tilde{c}(s,a)|$ \\
(weighted) & $\vec{c}_1, \ldots, \vec{c}_k \in \mathbb{R}^{|\S||\A|}$ & $\tilde{\vec{c}} = n_1 \vec{c_1} + \ldots + n_k \vec{c_k}$ \\
 & $n_1, \ldots, n_k \in \mathbb{R}$ &
\\ \cmidrule(ll){1-1} \cmidrule(ll){2-2} \cmidrule(ll){3-3}
Multi-objective MDPs & $f(\vec{d}) = \max_{i \in \{1, \ldots, k\}} \vec{d}^\top \vec{c}_i$ & $ \max_{i \in \{1, \ldots, k\}} L_i$ \\
(max) & $\vec{c}_1, \ldots, \vec{c}_k \in \mathbb{R}^{|\S||\A|}$ & $L_i = \max_{s,a} |c_i(s,a)|$ \\
\\ \cmidrule(ll){1-1} \cmidrule(ll){2-2} \cmidrule(ll){3-3}
Multi-objective MDPs & $f(\vec{d}) = \min_{i \in \{1, \ldots, k\}} \vec{d}^\top \vec{c}_i$ & $\max_{i \in \{1, \ldots, k\}} L_i$ \\
(min) & $\vec{c}_1, \ldots, \vec{c}_k \in \mathbb{R}^{|\S||\A|}$ & $L_i = \max_{s,a} |c_i(s,a)|$ \\
\bottomrule
\end{tabular}
\end{table}

\paragraph{Standard MDP ($f(\vec{d}) = \vec{c}^\top \vec{d}$)}
It holds that
\begin{equation*}
    |f(\vec{d}_1) - f(\vec{d}_2)| = |\vec{c}^\top (\vec{d}_1 - \vec{d}_2)| \le \sum_{s,a} |c(s,a)| |d_1(s,a) - d_2(s,a)| \le \max_{s,a} |c(s,a)| \| \vec{d}_1 - \vec{d}_2\|_1.
\end{equation*}

\paragraph{Maximum state entropy exploration ($f(\vec{d}) = \vec{d}^\top \log(\vec{d})$)}
We assume $\vec{d}$ is lower bounded by $\epsilon$, i.e., $d(s,a) \ge\epsilon$ with $0 < \epsilon < e^{-2}$ for all $s \in \S, a \in \A$. We let $f(\vec{d}) = \sum_{s,a} g(d(s,a))$, for $g(x) = x \log(x)$. We note that, $g'(x) = \log(x) + 1$ and it holds for any $x \in [\epsilon,1]$ that $|g'(x)| \le |\log(\epsilon) + 1|$. Thus, for any $x_1, x_2 \in [\epsilon,1]$ we have that

\begingroup
\allowdisplaybreaks
\begin{align*}
    | g(x_1) - g(x_2)| &= \left| \int_{x_2}^{x_1} g'(x) dx \right| = \left| \int_{\min\{x_1,x_2\}}^{\max\{x_1,x_2\}} g'(x) dx \right|\\
    &\le \int_{\min\{x_1,x_2\}}^{\max\{x_1,x_2\}} \left|g'(x) \right| dx \le \int_{\min\{x_1,x_2\}}^{\max\{x_1,x_2\}} \left|\log(\epsilon) + 1 \right| dx \\
    &= \left|\log(\epsilon) + 1 \right| |x_1 - x_2|.
\end{align*}
\endgroup

Thus, for any $\vec{d}_1, \vec{d}_2 \in \Delta(\S\times \A)$ lower bounded by $0 < \epsilon < e^{-2}$, it holds that
\begingroup
\allowdisplaybreaks
\begin{align*}
    |f(\vec{d}_1) - f(\vec{d}_1)| &= \left|\sum_{s,a} \left( g(d_1(s,a)) -g(d_2(s,a)) \right) \right| \\
    &\overset{(a)}{\le} \sum_{s,a} \left| g(d_1(s,a)) -g(d_2(s,a)) \right| \\
    &\le \sum_{s,a} \left|\log(\epsilon) + 1 \right| |d_1(s,a) - d_2(s,a)| \\
    &= \left|\log(\epsilon) + 1 \right| \| \vec{d}_1 - \vec{d}_2 \|_1
\end{align*}
\endgroup
were (a) follows from the triangular inequality.

\paragraph{Imitation learning ($f(\vec{d}) = \| \vec{d} - \vec{d}_{\pi_b}\|_2^2$)}
It holds that $\nabla f(\vec{d}) = 2 (\vec{d}- \vec{d}_{\pi_b})$. Now,
\begin{equation*}
    \max_{\vec{d} \in \Delta(\S \times \A) } \| \nabla f(\vec{d}) \|_1 = 2 \max_{\vec{d} \in \Delta(\S \times \A) } \| \vec{d} - \vec{d}_{\pi_b} \|_1 \le 2 \max_{\vec{d}_1, \vec{d}_2 \in \Delta(\S \times \A) } \| \vec{d}_1 - \vec{d}_2 \|_1 = 4.
\end{equation*}
Since the function $f$ is continuous and differentiable over the simplex, which is compact, it holds that $L = 4$ is a valid Lipschitz constant as it corresponds to an upper bound on the maximum magnitude of the gradient of $f$ over $\Delta(\S \times \A)$.

\paragraph{Multi-objective MDP, weighted ($f(\vec{d}) = n_1 \vec{d}^\top \vec{c}_1 + \ldots + n_k\vec{d}^\top \vec{c}_k$)}
We note that $f$ can be equivalently rewritten as $f(d) = \vec{d}^\top \tilde{\vec{c}}$, where $\tilde{\vec{c}} = n_1 \vec{c_1} + \ldots + n_k \vec{c_k}$. Therefore
\begin{equation*}
    |f(\vec{d}_1) - f(\vec{d}_2)| \le \max_{s,a} |\tilde{c}(s,a)| \| \vec{d}_1 - \vec{d}_2\|_1.
\end{equation*}

\paragraph{Multi-objective MDP, max ($f(\vec{d}) = \max_{i \in \{1, \ldots, k\}} \vec{d}^\top \vec{c}_i$)}
The result follows by noting that letting $L = \max_{i \in \{1, \ldots, k\}} L_i$ yields a valid Lipschiz constant, where $L_i = \max_{s,a} |c_i(s,a)|$, since we are taking the maximum of Lipschitz functions with Lipschitz constants $L_1, \ldots, L_k$.

\paragraph{Multi-objective MDP, min ($f(\vec{d}) = \min_{i \in \{1, \ldots, k\}} \vec{d}^\top \vec{c}_i$)}
It holds that
\begin{equation*}
    |f(\vec{d}_1) - f(\vec{d}_2)| = \left|\min_{i \in \{1, \ldots, k\}} \vec{c}_i^\top \vec{d}_1 - \min_{i \in \{1, \ldots, k\}} \vec{c}_i^\top \vec{d}_2 \right| \le \max_{i \in \{1, \ldots, k\}} |\vec{c}_i^\top \vec{d}_1 - \vec{c}_i^\top \vec{d}_2| \le L \| \vec{d}_1 - \vec{d}_2 \|_1,
\end{equation*}
where  $L = \max_{i \in \{1, \ldots, k\}} L_i$ and $L_i = \max_{s,a} |c_i(s,a)|$.

%
%
%
%

\section{Supplementary materials for Sec.~\ref{sec:solving_GUMDPs_with_online_planning}}
\label{appendix:B}

\subsection{Proof of Proposition~\ref{proposition:opt_gap_decomposition}}

\begin{lemma}
    \label{lemma:upper_bound_finite_vs_infinite_omega}
    For any $\omega \in \Omega$, $\pi \in \Pi_\textnormal{NM}$ and $H \in \mathbb{N}$ it holds that
    $|\rvar{f}^\pi(\omega) - \rvar{f}^{\pi}_H(\omega)| \le 2 L_f \gamma^H.$
\end{lemma}
\begin{proof}
    For any $\omega \in \Omega$, $\pi \in \Pi_\text{NM}$ and $H \in \mathbb{N}$ it holds that
    \begingroup
    \allowdisplaybreaks
    \begin{align*}
        |\rvar{f}^\pi(\omega) - \rvar{f}^{\pi}_H(\omega)| &= \left| f(\rvec{d}^{\pi}(\omega)) - f(\rvec{d}^{\pi,H}(\omega)) \right| \\
        &\overset{\text{(a)}}{\le} L_f  \left\| \rvec{d}^{\pi}(\omega)) - \rvec{d}^{\pi,H}(\omega)) \right\|_1 \\
        &\overset{\text{(b)}}{=} L_f \left\| (1-\gamma) \sum_{t=0}^\infty \gamma^t \rvec{d}^\pi_t(\omega) - \frac{(1-\gamma)}{1-\gamma^H} \sum_{t=0}^{H-1} \gamma^t \rvec{d}^{\pi}_t(\omega) \right\|_1 \\
        &= L_f \left\| \frac{(1-\gamma)}{1-\gamma^H} \sum_{t=0}^{H-1} \gamma^t \left( (1-\gamma^H) \rvec{d}^\pi_t(\omega) - \rvec{d}^\pi_t(\omega) \right) + (1-\gamma) \sum_{t=H}^\infty \gamma^t \rvec{d}^\pi_t(\omega) \right\|_1 \\
         &\overset{\text{(c)}}{\le} L_f \left( \frac{(1-\gamma)}{1-\gamma^H} \sum_{t=0}^{H-1} \gamma^t \left\|(1-\gamma^H) \rvec{d}^\pi_t(\omega) - \rvec{d}^\pi_t(\omega) \right\|_1 + (1-\gamma) \sum_{t=H}^\infty \gamma^t \| \rvec{d}^\pi_t(\omega) \|_1 \right) \\
        &= L_f \frac{(1-\gamma)}{1-\gamma^H} \gamma^H \sum_{t=0}^{H-1} \gamma^t \left\| \rvec{d}^\pi_t(\omega) \right\|_1 + L_f \gamma^H \\
        &= 2 L_f \gamma^H,
    \end{align*}
    \endgroup
    where: (a) is due to the $L_f$-Lipschitz assumption; in (b) we used $\rvec{d}^\pi(\omega) = (1-\gamma) \sum_{t=0}^\infty \gamma^t \rvec{d}^\pi_t(\omega)$ where $\rvar{d}^\pi_{t,(s,a)}(\omega) = \mathbf{1}(s_t=s, a_t = a)$ denotes the empirical occupancy induced by the trajectory $\omega$ at timestep $t$ and $\rvec{d}^{\pi,H}(\omega) = (1-\gamma) /(1-\gamma^H)\sum_{t=0}^{H-1} \gamma^t \rvec{d}^\pi_t(\omega)$. Step (c) follows from the triangular inequality.
\end{proof}

\begin{lemma} \label{lemma:upper_bound_truncated_non_truncated}
    For any policy $\pi \in \Pi_\textnormal{NM}$ and $H \in \mathbb{N}$ it holds that
    $\left| \rho(\rvar{f}^\pi_H) - \rho\left(\rvar{f}^\pi\right) \right| \le 2 L_f \gamma^H.$
\end{lemma}
\begin{proof}
    From Lemma~\ref{lemma:upper_bound_finite_vs_infinite_omega} we can infer that, for any $\omega \in \Omega$, $\rvar{f}^\pi(\omega) - 2 L_f \gamma^H \le \rvar{f}^{\pi}_H(\omega) \le \rvar{f}^\pi(\omega) + 2 L_f \gamma^H$. Thus, from the monotonicity of $\rho$ we have that $\rho(\rvar{f}^\pi - 2 L_f \gamma^H) \le \rho(\rvar{f}^{\pi}_H) \le \rho(\rvar{f}^\pi + 2 L_f \gamma^H)$. Equivalently, $\rho(\rvar{f}^\pi) - 2 L_f \gamma^H \le \rho(\rvar{f}^{\pi}_H) \le \rho(\rvar{f}^\pi) + 2 L_f \gamma^H$, and the result follows.
\end{proof}

\begin{lemma}
    \label{lemma:regret_of_optimal_truncated_policy}
    If $\pi^*_H = \argmin_{\pi \in \Pi_\textnormal{NM}} \rho(\rvar{f}^\pi_H)$, then it holds that
    $\mathrm{OptGap}(\pi^*_H ) \le 4  L_f \gamma^H$.
\end{lemma}
\begin{proof}
    As shown in Lemma~\ref{lemma:upper_bound_truncated_non_truncated}, $\left| \rho(\rvar{f}^\pi_H) - \rho\left(\rvar{f}^\pi\right) \right| \le 2 L_f \gamma^H$, for arbitrary $\pi \in \Pi_\text{NM}$. From such inequality, we can infer that $\rho\left(\rvar{f}^\pi_H\right) - 2 L_f\gamma^H \le \rho\left(\rvar{f}^\pi\right), \forall \pi \in \Pi_\text{NM}$, i.e., function $\rho\left(\rvar{f}^\pi_H\right) - 2 L_f\gamma^H$ lower bounds function $\rho\left(\rvar{f}^\pi\right)$. Let $\pi^*_H = \argmin_{\pi \in \Pi_\text{NM}} \rho\left(\rvar{f}^\pi_H\right)$.
    It holds that
    \begin{equation}
        \rho\left(\rvar{f}^{\pi^*_H}_H\right) - 2 L_f\gamma^H = \min_\pi \rho\left(\rvar{f}^\pi_H\right) - 2 L_f\gamma^H \overset{\text{(a)}}{\le} \min_\pi \rho\left(\rvar{f}^\pi\right) \overset{\text{(b)}}{\le} \rho\left(\rvar{f}^{\pi^*_H}\right), \label{eq:lemma:regret_of_optimal_truncated_policy:chain_of_ineqs}
    \end{equation}
    where (a) follows from the fact that $\rho\left(\rvar{f}^\pi_H\right) - 2 L_f\gamma^H$ lower bounds $\rho\left(\rvar{f}^\pi\right)$; and (b) from the fact that $\min_\pi \rho\left(\rvar{f}^\pi\right) \le \rho\left(\rvar{f}^{\pi'}\right), \forall \pi'$ (from the definition of a minimum). Finally, we note that 
    \begin{align*}
        \rho\left(\rvar{f}^{\pi^*_H}\right) - \left(\rho\left(\rvar{f}^{\pi^*_H}_H\right) - 2 L_f\gamma^H \right) &= \rho\left(\rvar{f}^{\pi^*_H}\right) - \rho\left(\rvar{f}^{\pi^*_H}_H\right) + 2 L_f \gamma^H \\
        &\le \left| \rho\left(\rvar{f}^{\pi^*_H}\right) - \rho\left(\rvar{f}^{\pi^*_H}_H\right) \right| + 2 L_f \gamma^H\\
        &\le 4 L_f\gamma^H.
    \end{align*}
    The above and \eqref{eq:lemma:regret_of_optimal_truncated_policy:chain_of_ineqs} imply that
    $$ \mathrm{OptGap}(\pi^*_H) = \rho\left(\rvar{f}^{\pi^*_H}\right) - \min_{\pi'} \rho\left(\rvar{f}^{\pi'}\right) \le 4 L_f \gamma^H.$$
\end{proof}

\paragraph{Proof of Proposition \ref{proposition:opt_gap_decomposition}}
\begin{proof}
    Let $\pi^*_H = \argmin_{\pi \in \Pi_\text{NM}} \rho(\rvar{f}^\pi_H)$, i.e., $\pi^*_H$ is optimal with respect to the truncated objective. It holds that,
    \begingroup
    \allowdisplaybreaks
    \begin{align*}
        \mathrm{OptGap}(\pi) &= \rho\left(\rvar{f}^\pi\right) - \min_{\pi' \in \Pi_{\text{NM}}} \rho\left(\rvar{f}^{\pi'}\right) \\
        &= \left| \rho\left(\rvar{f}^\pi\right) - \min_{\pi' \in \Pi_{\text{NM}}} \rho\left(\rvar{f}^{\pi'}\right) \right|\\
        &\overset{\text{(a)}}{\le} \left| \rho\left(\rvar{f}^\pi\right) - \rho(\rvar{f}^{\pi_H^*}) \right| + \left| \rho(\rvar{f}^{\pi_H^*}) - \min_{\pi' \in \Pi_{\text{NM}}} \rho\left(\rvar{f}^{\pi'}\right) \right|\\
        &\overset{\text{(b)}}{\le} \left| \rho\left(\rvar{f}^\pi\right) - \rho(\rvar{f}^{\pi_H^*}) \right| + 4 L_f \gamma^H \\
        &\overset{\text{(c)}}{\le} \left| \rho\left(\rvar{f}^\pi\right) - \rho\left(\rvar{f}^\pi_H\right) \right| + \left| \rho\left(\rvar{f}^\pi_H\right) - \rho(\rvar{f}^{\pi_H^*}) \right| + 4 L_f \gamma^H \\
        &\overset{\text{(d)}}{\le} 2 L_f \gamma^H + \left| \rho\left(\rvar{f}^\pi_H\right) - \rho(\rvar{f}^{\pi_H^*}) \right| + 4 L_f \gamma^H \\
        &\overset{\text{(e)}}{\le} 2 L_f \gamma^H + \left| \rho\left(\rvar{f}^\pi_H\right) - \rho(\rvar{f}^{\pi_H^*}_H) \right| + \left| \rho(\rvar{f}^{\pi_H^*}_H) -\rho(\rvar{f}^{\pi_H^*}) \right| + 4 L_f \gamma^H \\
        &\overset{\text{(f)}}{\le} 2 L_f \gamma^H + \left| \rho\left(\rvar{f}^\pi_H\right) - \rho(\rvar{f}^{\pi_H^*}_H) \right| + 2 L_f \gamma^H + 4 L_f \gamma^H \\
        &= \rho(\rvar{f}^\pi_H)  - \min_{\pi' \in \Pi_\text{NM}} \left\{ \rho(\rvar{f}^{\pi'}_H) \right\} + 8 L_f \gamma^H,
    \end{align*}
    \endgroup
    where (a) follows from adding and subtracting $\rho(\rvar{f}^{\pi_H^*})$ and applying the triangular inequality; (b) follows from Lemma \ref{lemma:regret_of_optimal_truncated_policy}; (c) follows from adding and subtracting $\rho\left(\rvar{f}^\pi_H\right)$ and applying the triangular inequality; (d) follows from Lemma \ref{lemma:upper_bound_truncated_non_truncated}; (e) follows from adding and subtracting $\rho(\rvar{f}^{\pi_H^*}_H)$ and applying the triangular inequality; and (f) follows from Lemma \ref{lemma:upper_bound_truncated_non_truncated}.
\end{proof}

\subsection{Proof of Lemma~\ref{lemma:equivalence_in_distribution}}
\begin{proof}
    Random variable $\rvar{f}_H^\pi$, as defined in \eqref{eq:truncated_f_random_var_definition}, is associated with the probability space $(\Omega, \F, \mathbb{P}_\pi)$ and random variable $\rvar{J}^{\pi_\text{O}}_\text{O}$, as introduced in Sec.~\ref{sec:occupancy_mdp}, is associated with the probability space $(\Omega_\text{O}, \F_\text{O}, \mathbb{P}_{\pi_\text{O}}^\text{O})$. Random variables $\rvar{f}_H^\pi$ and $\rvar{J}^{\pi_\text{O}}_\text{O}$ are equal in distribution if, for any $y \in \Y$,
    $$\PP[\pi]{\rvar{f}_H^\pi \le y} = \mathbb{P}_{\pi_\text{O}}^\text{O}[\rvar{J}^{\pi_\text{O}}_\text{O} \le y].$$
    
    We start by noting that, for any trajectory $\omega_\text{O} = (\{s_0,\vec{o}_0\}, a_0, \{s_1,\vec{o}_1\}, a_1, \ldots, \{s_{H}, \vec{o}_H\}) \in \Omega_\text{O}$,
    \begingroup
    \allowdisplaybreaks
    \begin{align*}
        \mathbb{P}_{\pi_\text{O}}^\text{O}[\omega_\text{O}] &=  p_{0,\text{O}}(\{s_0,\vec{o}_0\}) \cdot \pi_\text{O}(a_0|\{s_0,\vec{o}_0\}) \cdot P_\text{O}^{a_0}(\{\rvar{s}_0,\rvec{o}_0\}, \{\rvar{s}_1,\rvec{o}_1\}) \cdot \ldots \\
        &\qquad \qquad \cdot \pi_\text{O}(a_{H-1}|\{s_{H-1},\vec{o}_{H-1}\}) \cdot P_\text{O}^{a_{H-1}}(\{s_{H-1},\vec{o}_{H-1}\}, \{s_{H},\vec{o}_{H}\}).\\
        &\overset{(a)}{=} p_{0}(s_0) \cdot \mathbf{1}(\vec{o}_0 = [0,\ldots,0]) \cdot \pi_\text{O}(a_0|\{s_0,\vec{o}_0\}) \cdot P^{a_0}(s_0,s_1) \cdot \mathbf{1}(\vec{o}_1 = \sigma(s_0,\vec{o}_0,a_0)) \cdot \ldots\\
        &\qquad \qquad \cdot \pi_\text{O}(a_{H-1}|\{s_{H-1},\vec{o}_{H-1}\}) \cdot P^{a_{H-1}}(s_{H-1}, s_{H}) \cdot \mathbf{1}(\vec{o}_H = \sigma(s_{H-1}, \vec{o}_{H-1}, a_{H-1}))\\
        &\overset{(b)}{=} p_{0}(s_0) \cdot \mathbf{1}(\vec{o}_0 = [0,\ldots,0]) \cdot \pi(a_0|h_0) \cdot P^{a_0}(s_0,s_1) \cdot \mathbf{1}(\vec{o}_1 = \sigma(s_0,\vec{o}_0,a_0)) \cdot \ldots\\
        &\qquad \qquad \cdot \pi(a_{H-1}|h_{H-1}) \cdot P^{a_{H-1}}(s_{H-1}, s_{H}) \cdot \mathbf{1}(\vec{o}_H = \sigma(s_{H-1}, \vec{o}_{H-1}, a_{H-1})) \\
        &\overset{(c)}{=} \PP[\pi]{\omega} \cdot P^{a_{H-1}}(s_{H-1}, s_{H}) \cdot \mathbf{1}(\vec{o}_0 = [0,\ldots,0]) \cdot \mathbf{1}(\vec{o}_1 = \sigma(s_0,\vec{o}_0,a_0)) \cdot \ldots \\
        &\qquad \qquad \cdot \mathbf{1}(\vec{o}_H = \sigma(s_{H-1}, \vec{o}_{H-1}, a_{H-1})),
    \end{align*}
    \endgroup
    where in (a) we note that component $\vec{o}$ of the state is initialized as a zero vector and then deterministically evolves according to $\sigma$; any sequence of $\vec{o}$-vectors that does not evolve according to $\sigma$ has zero probability under probability measure $\mathbb{P}_{\pi_\text{O}}^\text{O}$. In (b) we used the fact that any stationary policy $\pi_\text{O} \in \Pi_\text{S}$ for $\M_\text{O}$ can be mapped to a particular non-Markovian policy $\pi \in \Pi_\text{NM}$ in $\M_f$. In (c) we recall that, for $\omega = (s_0, a_0, s_1, a_1, \ldots, s_{H-1}, a_{H-1})$,
    $\PP[\pi]{\omega} = p_0(s_0) \cdot \pi(a_0|h_0) \cdot P^{a_0}(s_0, s_1) \cdot \pi(a_1|h_1) \cdot P^{a_1}(s_1, s_2) \ldots P^{a_{H-2}}(s_{H-2}, s_{H-1}) \cdot \pi(a_{H-1}|h_{H-1})$.

    Now, for any $y \in \Y$ and stationary policy $\pi_\text{O} \in \Pi_\text{S}$, it holds that
    \begingroup
    \allowdisplaybreaks
    \begin{align*}
        \mathbb{P}_{\pi_\text{O}}^\text{O}[\rvar{J}^{\pi_\text{O}}_\text{O} \le y] &= \mathbb{P}_{\pi_\text{O}}^\text{O}[\{\omega_\text{O} : \rvar{J}^{\pi_\text{O}}_\text{O}(\omega_\text{O}) \le y\}] \\
        &= \sum_{\omega_\text{O} \in \Omega_\text{O}} \mathbb{P}_{\pi_\text{O}}^\text{O}[\omega_\text{O}] \mathbf{1}\left(\rvar{J}^{\pi_\text{O}}_\text{O}(\omega_\text{O}) \le y \right) \\
        &= \sum_{\omega_\text{O} \in \Omega_\text{O}} \mathbb{P}_{\pi_\text{O}}^\text{O}[\omega_\text{O}] \mathbf{1}\left(f\left(\frac{1-\gamma}{1-\gamma^H} \vec{o}_H\right) \le y \right) \\
        &= \sum_{\omega_\text{O} \in \Omega_\text{O}} \PP[\pi]{\omega} \cdot P^{a_{H-1}}(s_{H-1}, s_{H}) \cdot \mathbf{1}(\vec{o}_0 = [0,\ldots,0]) \cdot \mathbf{1}(\vec{o}_1 = \sigma(s_0,\vec{o}_0,a_0)) \cdot \ldots \\
        &\qquad \qquad \cdot \mathbf{1}(\vec{o}_H = \sigma(s_{H-1}, \vec{o}_{H-1}, a_{H-1})) \mathbf{1}\left(f\left(\frac{1-\gamma}{1-\gamma^H} \vec{o}_H\right) \le y \right) \\
        &\overset{(a)}{=} \sum_{\omega_\text{O} \in \Omega_\text{O}} \PP[\pi]{\omega} \cdot P^{a_{H-1}}(s_{H-1}, s_{H}) \cdot \mathbf{1}(\vec{o}_0 = [0,\ldots,0]) \cdot \mathbf{1}(\vec{o}_1 = \sigma(s_0,\vec{o}_0,a_0)) \cdot \ldots \\
        &\qquad \qquad \cdot \mathbf{1}(\vec{o}_H = \sigma(s_{H-1}, \vec{o}_{H-1}, a_{H-1})) \mathbf{1}\left(f\left(\rvec{d}^{\pi,H}(\omega)\right) \le y \right) \\
        &\overset{(b)}{=} \sum_{\omega \in \Omega} \PP[\pi]{\omega} \mathbf{1}\left(f\left(\rvec{d}^{\pi,H}(\omega)\right) \le y \right) \sum_{\vec{o}_0, \vec{o}_1, \ldots, \vec{o}_H \in \O} \sum_{s_H \in \S} P^{a_{H-1}}(s_{H-1}, s_{H}) \cdot \\
        &\qquad \qquad \mathbf{1}(\vec{o}_0 = [0,\ldots,0]) \cdot \mathbf{1}(\vec{o}_1 = \sigma(s_0,\vec{o}_0,a_0)) \cdot \ldots \cdot \mathbf{1}(\vec{o}_H = \sigma(s_{H-1}, \vec{o}_{H-1}, a_{H-1})) \\
        &\overset{(c)}{=} \sum_{\omega \in \Omega} \PP[\pi]{\omega} \mathbf{1}\left(f\left(\rvec{d}^{\pi,H}(\omega)\right) \le y \right) \\
        &= \sum_{\omega \in \Omega} \PP[\pi]{\omega} \mathbf{1}\left(\rvar{f}_H^\pi\left(\omega\right) \le y \right) \\
        &= \sum_{\omega \in \Omega} \PP[\pi]{\{\omega : \rvar{f}_H^\pi\left(\omega\right) \le y\}} \\
        &= \PP[\pi]{\rvar{f}_H^\pi \le y},
    \end{align*}
    \endgroup
    where in (a) we noted that, for any $\omega_\text{O} \in \Omega_\text{O}$, $\mathbf{1}\left(f\left(\frac{1-\gamma}{1-\gamma^H} \vec{o}_H\right) \le y \right) = \mathbf{1}\left(f\left(\rvec{d}^{\pi,H}(\omega)\right) \le y \right)$. In (b), we split the sum over $\omega_\text{O} \in \Omega_\text{O}$ as a sum over $\omega \in \Omega$, a sum over each possible vector $\vec{o} \in \O$ across all timesteps, and a sum over the final state $s_H \in \S$ (not included in $\omega$). We also rearranged the sums by noting that some terms do not depend on some of the sums. In (c) we note that the inner sums over the $\vec{o}$-vectors and $s_H$ equal one.
\end{proof}

\subsection{Proof of Theorem~\ref{theo:gumdp_equivalence_occupancy_MDP}}
\begin{proof}
    From Lemma~\ref{lemma:equivalence_in_distribution}, for any horizon $H \in \mathbb{N}$ and policy $\pi \in \Pi_\text{NM}$, it holds that $\rvar{f}_H^\pi \overset{D}{=} \rvar{J}_\text{O}^{\pi_\text{O}}$, where $\pi_\text{O}$ is the stationary policy for $\M_\text{O}$ associated with the non-Markovian policy $\pi$ for $\M_f$. Thus, for any risk-measure $\rho : \L \rightarrow \mathbb{R}$, it holds that $\rho(\rvar{f}_H^\pi) = \rho(\rvar{J}_\text{O}^{\pi_\text{O}})$. Also, from Lemma~\ref{lemma:one_to_one_mapping_histories_states} there exists a one-to-one mapping between non-Markovian policies for $\M_f$ and stationary policies for $\M_\text{O}$ since every state in $\M_\text{O}$ is uniquely associated with a particular history in $\M_f$ (and vice versa). Given these two results, we have that, for any $\pi \in \Pi_\text{NM}$,
    \begin{equation*}
        \mathrm{OptGap}_H(\pi) = \rho\left(\rvar{f}^\pi_H\right) -  \min_{\pi' \in \Pi_{\text{NM}}} \rho\left(\rvar{f}^{\pi'}_H\right) = \rho\left(\rvar{J}_\text{O}^{\pi_\text{O}}\right) - \min_{\pi'_\text{O} \in \Pi_\text{S}} \rho\left(\rvar{J}_\text{O}^{\pi'_\text{O}}\right),
    \end{equation*}
    and the conclusion follows.
\end{proof}

\clearpage
\section{\textsc{ERM-MCTS} pseudocode}
\label{appendix:ERM-MCTS-pseudocode}
We display in~\Cref{algo:erm_mcts} the pseudocode for the \textsc{ERM-MCTS} algorithm.

\begin{algorithm}
\caption{$\text{ERM}_\beta$-MCTS for finite-horizon MDPs with terminal costs.}
    \textbf{Inputs:} $s_0$ (root node), $H$ (horizon), $n$ (number of MCTS iterations), and $\{\theta_t\}_{t \in \{1, \ldots, H-1\}}$ (exploration constants). \\
    \textbf{Initialization:} $\X_{(s_t,a)} = [\;], \; \forall t \in \{0, \ldots, H-1\}, s_t \in \S, a \in \A$ (Initialize list to store sampled terminal costs starting from $(s_t,a)$). $N(s_t) = 0, N(s_t,a) = 0, \; \forall t \in \{0, \ldots, H-1\}, s \in \S, a \in \A$ (Initialize visitation counters).
    \begin{algorithmic}[1]
        \For{\( j \in n\)}
            \For{\( t \in \{0, \ldots, H-1\}\)}
            \If{$N(\rvar{s}_t,a) = 0$ for some $a \in \A$}
                $\rvar{a}_t = a$.
            \Else
                \State Select an action according to:
                \State $\quad \rvar{a}_t \in \argmin_{a \in \A} \Bigg\{  \hat{\rho}_{(\rvar{s}_t,a),N(\rvar{s}_t,a)} - \theta_t \sqrt{\frac{\sqrt{N(\rvar{s}_t)}}{N(\rvar{s}_t,a)}}\Bigg\}$,
                \State $\quad$ where $\hat{\rho}_{(\rvar{s}_t,a),N(\rvar{s}_t,a)} = \frac{1}{\beta} \ln\left(\frac{1}{N(\rvar{s}_t,a)} \sum_{i=1}^{N(\rvar{s}_t,a)} \exp \left(\beta \rvar{x}_i^{(\rvar{s}_t,a)} \right)\right).$
            \EndIf
            \State Sample and transition to next state $\rvar{s}_{t+1} \sim P^{\rvar{a}_t}_\text{O}(\cdot|\rvar{s}_t)$.
            \EndFor
            \State At depth $H$ observe terminal cost $c_\text{O}(\rvar{s}_H)$.
            \For{\( t \in \{H-1, \ldots, 0\}\)}
                \State Update visitation counters:
                \State $\quad \quad N(\rvar{s}_t) = N(\rvar{s}_t) + 1$.
                \State $\quad \quad N(\rvar{s}_t,\rvar{a}_t) = N(\rvar{s}_t,\rvar{a}_t) + 1$.
                \State Store sampled terminal cost:
                \State $\quad \quad \text{Store } \rvar{x}_{N(\rvar{s}_t,\rvar{a}_t)}^{(\rvar{s}_t,\rvar{a}_t)} = c_\text{O}(\rvar{s}_H)$ in $\X_{(\rvar{s}_t,\rvar{a}_t)}$.
            \EndFor
        \EndFor
    \label{algo:erm_mcts}
    \end{algorithmic}
\end{algorithm}

\section{Experimental Results}
\label{appendix:experimental-results}

Our code is available at \url{https://github.com/gh0stwin/risk-aware-gumdp}.

\subsection{Baselines}
In the context of standard MDPs, the \textsc{ERM-BI} baseline exploits the dynamic programming decompositions of the ERM put forth by \citep{hau_2023}. In particular, \citet{hau_2023} show that the optimal value function $V^* = \{V^*_{t}\}_{t \in \{0, \ldots, H\}}$ and the optimal policy $\pi^* = \{\pi^*_t\}_{t \in \{0, \ldots, H-1\}}$ satisfy, for all $s \in \S$,
\begin{align*}
    V_t^*(s) &= \min_{a \in \A}\left\{\text{ERM}_{\beta \gamma^t}\left( c_t(s,a) + \gamma V_{t+1}^*(\rvar{s}')\right) \right\}, \; \forall t \in \{0, \ldots, H-1\}, \; V_{H}^*(s) = c_H(s) \label{eq:erm_bellman_optimality equations}, \\
    \pi^*_t(s) &= \argmin_{a \in \A}\left\{\text{ERM}_{\beta \gamma^t}\left( c_t(s,a) + \gamma V_{t+1}^*(\rvar{s}')\right) \right\}, \; \forall t \in \{0, \ldots, H-1\}.
\end{align*}

In the context of standard MDPs, we perform backward induction from timestep $t=H$ backwards until timestep $t=0$ to extract the optimal risk-aware policy. We also note that the occupancy MDP is undiscounted (i.e., $\gamma = 1$) and, hence, $\beta \gamma^t = \beta$ for all timesteps.

\subsection{Hyperparameters}
For implementation purposes, we let the \textsc{ERM-MCTS} parameters $\theta_t = \theta$ for all $ t \in \{0, \ldots, H-1\}$, see~\Cref{sec:on-planning}. We display in~\Cref{tab:appendix:illustrative-gumdps-hyperparams,tab:appendix:grid-gumdps-hyperparams} the hyperparameters used in our experiments, where $N$ denotes the number of independent runs of \textsc{ERM-MCTS}, $n$ is the number of MCTS iterations per timestep for \textsc{ERM-MCTS}, and $\theta$ is the \textsc{ERM-MCTS} exploration constant. Under MDPs, we also run $N$ independent runs of the oracle baseline \textsc{ERM-BI}.

\begin{table}[h]
\centering
\resizebox{0.8\columnwidth}{!}{%
\begin{tabular}{@{}cccccc@{}}
\toprule
\textbf{Environment} & \textbf{\begin{tabular}[c]{@{}c@{}}$\gamma$\\ (Discount)\end{tabular}} & \textbf{\begin{tabular}[c]{@{}c@{}}$H$\\ (Horizon)\end{tabular}} & \textbf{\begin{tabular}[c]{@{}c@{}}$N$\\ (Num. runs)\end{tabular}} & \textbf{\begin{tabular}[c]{@{}c@{}}$n$\\ (MCTS iter.)\end{tabular}} & \textbf{\begin{tabular}[c]{@{}c@{}}$\theta$\\ (Expl. const.)\end{tabular}} \\ \midrule
\multicolumn{1}{c}{\textbf{MDP}} & \multicolumn{1}{c}{0.9} & \multicolumn{1}{c}{20} & \multicolumn{1}{c}{100} & \multicolumn{1}{c}{500} & \multicolumn{1}{c}{1} \\ \midrule
\multicolumn{1}{c}{\textbf{MSEE}} & \multicolumn{1}{c}{0.9} & \multicolumn{1}{c}{20} & \multicolumn{1}{c}{100} & \multicolumn{1}{c}{500} & \multicolumn{1}{c}{1} \\ \midrule
\multicolumn{1}{c}{\textbf{IL}} & \multicolumn{1}{c}{0.9} & \multicolumn{1}{c}{20} & \multicolumn{1}{c}{100} & \multicolumn{1}{c}{2 000} & \multicolumn{1}{c}{1} \\ \midrule
\multicolumn{1}{c}{\textbf{MO}} & \multicolumn{1}{c}{0.99} & \multicolumn{1}{c}{20} & \multicolumn{1}{c}{100} & \multicolumn{1}{c}{500} & \multicolumn{1}{c}{1} \\ \bottomrule
\end{tabular}%
}
\caption{Hyperparameters used for the illustrative environment experiments.}
\label{tab:appendix:illustrative-gumdps-hyperparams}
\end{table}

\begin{table}[h]
\centering
\resizebox{0.8\columnwidth}{!}{%
\begin{tabular}{@{}cccccc@{}}
\toprule
\textbf{Environment} & \textbf{\begin{tabular}[c]{@{}c@{}}$\gamma$\\ (Discount)\end{tabular}} & \textbf{\begin{tabular}[c]{@{}c@{}}$H$\\ (Horizon)\end{tabular}} & \textbf{\begin{tabular}[c]{@{}c@{}}$N$\\ (Num. runs)\end{tabular}} & \textbf{\begin{tabular}[c]{@{}c@{}}$n$\\ (MCTS iter.)\end{tabular}} & \textbf{\begin{tabular}[c]{@{}c@{}}$\theta$\\ (Expl. const.)\end{tabular}} \\ \midrule
\multicolumn{1}{c}{\textbf{MSEE}} & \multicolumn{1}{c}{0.99} & \multicolumn{1}{c}{200} & \multicolumn{1}{c}{128} & \multicolumn{1}{c}{1024} & \multicolumn{1}{c}{\(\sqrt{2}\)} \\ \midrule
\multicolumn{1}{c}{\textbf{IL}} & \multicolumn{1}{c}{0.99} & \multicolumn{1}{c}{99} & \multicolumn{1}{c}{128} & \multicolumn{1}{c}{1024} & \multicolumn{1}{c}{\(\sqrt{2}\)} \\ \midrule
\multicolumn{1}{c}{\textbf{MO}} & \multicolumn{1}{c}{0.99} & \multicolumn{1}{c}{40} & \multicolumn{1}{c}{128} & \multicolumn{1}{c}{1024} & \multicolumn{1}{c}{\(\sqrt{2}\)} \\ \midrule
\end{tabular}%
}
\caption{Hyperparameters used for the grid environment experiments.}
\label{tab:appendix:grid-gumdps-hyperparams}
\end{table}

%

\subsection{Illustrative environments}
We display in Fig.~\ref{fig:illustrative_envs} an overview of the illustrative environments. In the remainder of this section, we detail the GUMDPs used for each task, regarding their state space, action space, dynamics, and cost function used.

\subsubsection{Standard MDP (Fig.~\ref{fig:illustrative_envs}~(a))} \label{app:illus-env-mdp}
We employ a four-state MDP where the agent needs to tradeoff at the initial state ($s_0$) between:
\begin{enumerate*}
    \item a risky action ($a_0$) that can lead to a low-cost state ($s_2$), but with some probability the agent ends in a high-cost state ($s_3$); and
    \item a safe action ($a_1$) that deterministically leads to a medium-cost state ($s_1$).
\end{enumerate*}
The agent resets to the initial state with $10\%$ probability while not in the initial state. The cost function is given by $\vec{c} = [c(s_0),c(s_1),c(s_2),c(s_3)]$ where $c(s_0) = 0, \; c(s_1) = \frac{1}{4}, \; c(s_2) = \frac{1}{20}, \; c(s_3) = 1$. We let $f(\vec{d}) = \vec{d}^\top \vec{c}$.

\subsubsection{Maximum State Entropy Exploration (Fig.~\ref{fig:illustrative_envs}~(b))}
The task is similar to the motivating example from~\Cref{sec:intro}, where the agent aims to explore an environment as uniformly as possible but there is chance that the agent transitions to an absorbing state ($s_4$). The agent starts at state $s_1$. We let $f(\vec{d}) = \vec{d}^\top\log(\vec{d})$.

\subsubsection{Imitation Learning (Fig.~\ref{fig:illustrative_envs}~(a))}
The dynamics of this environment are the same as the standard MDP (\Cref{app:illus-env-mdp}). The agent aims to imitate the empirical occupancy induced by the trajectory of an agent that selected twice the risky action ($a_0$) under the standard MDP and then selected the safe action ($a_1$) for all the remainder timesteps. The trajectory to imitate had a rather \emph{lucky} outcome, as it never ended up in the absorbing state ($s_3$). Hence, the agent needs to trade off between imitating the behavior policy in states ($s_0$, $s_1$, $s_2$) and risking being absorbed into $s_3$, or only imitating the behavior policy in the \emph{less risky} states ($s_0$, $s_1$). To be precise, the empirical occupancy to imitate is $d_{\pi_b}(s_0, a_0)= 0.20605099$, $d_{\pi_b}(s_0, a_1)= 0.30175732$, $d_{\pi_b}(s_1, a_0)= 0.17054104$, $d_{\pi_b}(s_1, a_1)= 0.15004508$, $d_{\pi_b}(s_2, a_0)= 0.10245629$, $d_{\pi_b}(s_2, a_1)= 0.0829896$, $d_{\pi_b}(s_3, a_0)= 0.0$, $d_{\pi_b}(s_3, a_1)= 0.0$. We let $f(\vec{d}) = \| \vec{d} - \vec{d}_{\pi_b}\|_2^2$.

\subsubsection{Multi-Objective MDP (Fig.~\ref{fig:illustrative_envs}~(c))}
The environment is inspired by the FishWood environment \citep{roijers_2020}, where the agent needs to trade off between two cost functions. Each cost function penalizes different behaviors. The cost function $\vec{c_1}$ is defined as $c_1(s_0) = 0$, $c_1(s_1) = -1$, $c_1(s_2) = 0.5$, $c_1(s_3) = 0$, $c_1(s_4) = 0$. The cost function $\vec{c_2}$ is defined as $c_2(s_0) = 0$, $c_2(s_1) = 0$, $c_2(s_2) = 0$, $c_2(s_3) = -0.2$, $c_2(s_4) = 0.2$. The agent starts at state $s_0$. Any action at states $s_1$, $s_2$, $s_3$, and $s_4$ takes the agent deterministically back to the initial state. We consider three utility functions: 
\begin{enumerate*}[label=(\roman*)]
    \item \textit{weighted}, where $f(\vec{d}) = \vec{d}^\top \vec{c}_1 + \vec{d}^\top \vec{c}_2$;
    \item \textit{max}, where $f(\vec{d}) = \max(\vec{d}^\top \vec{c}_1, \vec{d}^\top \vec{c}_2)$; and
    \item \textit{min}, where $f(\vec{d}) = \min( \vec{d}^\top \vec{c}_1, 2 \cdot \vec{d}^\top \vec{c}_2)$.
\end{enumerate*}

\subsection{Grid Environments}
\subsubsection{Maximum State Entropy Exploration} \label{app:msee_grid}
\Cref{tab:grid-msee-params} contains the hyperparameters used in the environment employed in the MSEE task. The environment dynamics are described in~\Cref{sec:grid_msee}.

\begin{table}[t]
\centering
\caption{Grid environment parameters used in the MSEE and IL task. Grid positions are defined using 0-based indexing and \((0,0)\) position it at the top-left corner.}
\label{tab:grid-msee-params}
\begin{tabular}{c c}
\toprule
Parameter & Value \\
\midrule
Grid size & \(10\times10\) \\
Agent starting position & \((9,0)\) \\
Difficult terrain positions & \((6, 0),(6,1),(6,2),(6,3),(6,4),(8,3),(9,3),\)\\
& \((1,5),(1,6),(1,7),(1,8),(2,5),(2,6),(2,7),\)\\
& \((2,8),(7,5),(3,5),(3,9),(9,1),(5,2)\) \\
\(p_\text{trap}\) & \(0.1\) \\
\(p_\text{untrap}\) & \(0.01\) \\

\bottomrule
\end{tabular}
\end{table}

\subsection{Imitation Learning}
We reuse the environment and corresponding hyperparameters defined in the MSEE task (\Cref{app:msee_grid}). Additionally, we compute the occupancy of the behavioral policy \(\vec{d}_{\pi_b}\) after the agent performing the sequence of actions depicted in~\Cref{fig:il_beh_pol}. As already mentioned in~\Cref{sec:grid_il}, we ensure that the behavior policy never gets trapped inside difficult terrain squares. Under this approach we make sure that \textsc{ERM-MCTS} can model different behaviors (with distinct \emph{risk awareness}) by considering which difficult terrain squares to visit, if any.

\begin{figure}[t]
    \begin{center}
        \includegraphics[width=0.5\linewidth]{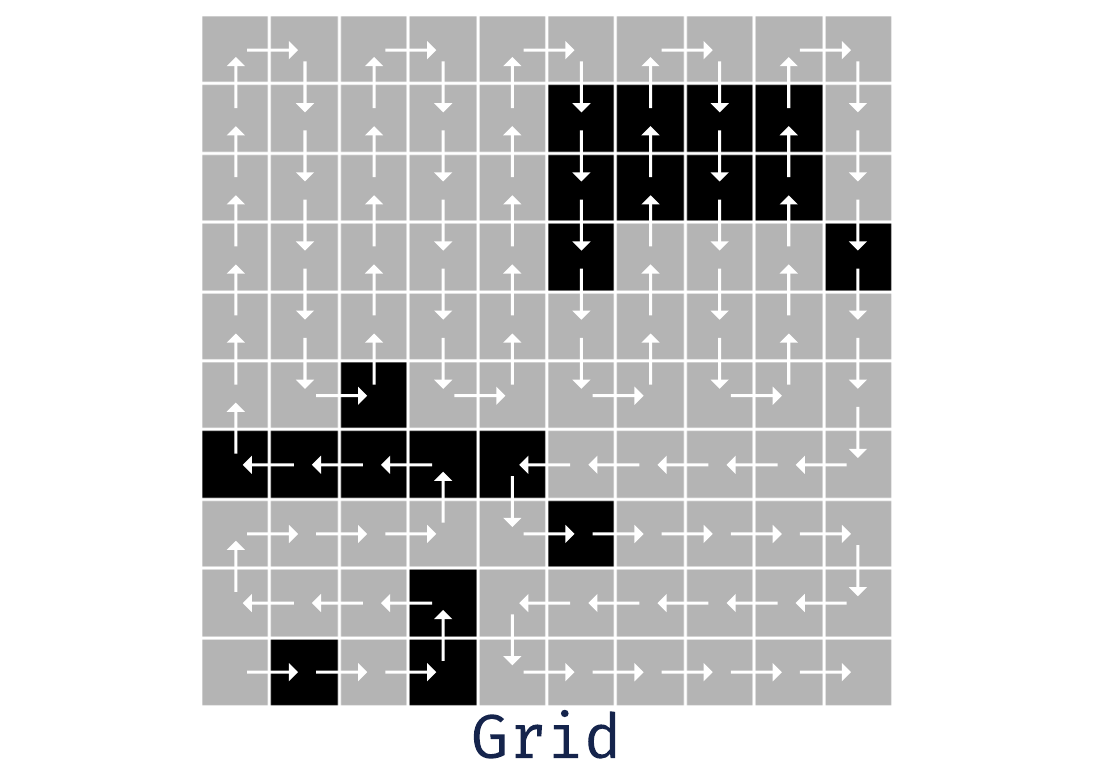}
    \end{center}
    \caption{IL (grid environment): Sequence of actions taken by the behavior policy \(\pi_b\). Agent starts in position \((9,0)\), assuming 0-based indexing and \((0,0)\) position it at the top-left corner.}
    \label{fig:il_beh_pol}
\end{figure}

\subsubsection{Multi-Objective} \label{app:mo-grid}
As stated in~\Cref{sec:grid-mo}, we use the resource-gathering environment~\citep{mo-rg} as our multi-objective task. Here, the agent has three objectives to maximize. The first two objectives address picking and delivering resources \(R_1\) and \(R_2\) to the home location, respectively. The last objective attends to the possibility of the agent getting defeated when stepping inside squares that contain enemies. Aiming to model the necessary information that is required for the reward functions, we define the environment state as a 5-tuple, \(s=\left((x,y),h_{R_1},h_{R_2},h_\text{start},h_\text{enemy}\right)\). Where \((x,y)\in\mathbb{N}^2\) holds the agent's position, \(h_{R_1},h_{R_2}\in\{0,1\}\) assign if the agent is carrying the corresponding resource, \(h_\text{start}\in\{0,1\}\) indicates whether the agent left the initial position, and \(h_\text{enemy}\in\{0,1\}\) indicates if a \emph{clash} with an enemy occurred. Regarding environment dynamics, the agent can move in all directions, plus we define a slippery condition where, with \(p_\text{slip}\) probability, the agent moves instead perpendicularly to the selected direction. Additionally, resources are collected when the agent enters the resource's square, setting \(h_r=1\), where \(r\in\{R_1,R_2\}\). Furthermore, when the agent steps inside a square that contains an enemy, with \(p_\text{defeat}\) probability, \(h_\text{enemy}\) becomes \(1\). In the case when \(h_\text{enemy}=1\), the agent loses immediately and transitions to an absorbing state until the end of the episode. Regarding reward modeling, we define three different reward functions \(\vec{r}_1,\ldots,\vec{r}_3\in\mathbb{R}^{|S||A|}\), where all entries are \(0\) except in specific states:
\begin{enumerate}[label=(\roman*)]
    \item For \(\vec{r}_1\), the states \(s\in\{((x_\text{home},y_\text{home}),1,0,1,0),((x_\text{home},y_\text{home}),1,1,1,0)\}\) receive \(1\) to compensate for delivering \(R_1\) to the agent's home position \((x_\text{home},y_\text{home})\).
    \item Analogously, \(\vec{r}_2\) accounts for the cases where \(R_2\) gets transported to  \((x_\text{home},y_\text{home})\), thus assigning the states \(s\in\{((x_\text{home},y_\text{home}),0,1,1,0),((x_\text{home},y_\text{home}),1,1,1,0)\}\) with reward \(1\).
    \item To account for when the agent loses as the consequence of an enemy clash, \(\vec{r}_3\) sets the reward to \(-1\) for any state \(s\) having \(h_\text{enemy}=1\).
\end{enumerate}
Subsequently, we combine all rewards by applying the following non-linear function, \(f(\vec{d})=\left(\vec{r}_1^\top\vec{d}\right)^{\frac{1}{2}}+\vec{r}_2^\top\vec{d}+\vec{r}_3^\top\vec{d}\). With \(f(\vec{d})\)\footnote{Transforming this objective to handle costs, we first set \(\vec{c}_i=-\vec{r}_i,\forall i\in\{1,\ldots,3\}\) and then use \(f(\vec{d})=\text{sign}(\vec{c}_1^\top\vec{d}){|\vec{c}_1^\top\vec{d}|}^{\frac{1}{2}}+\vec{c}_2^\top\vec{d}+\vec{c}_3^\top\vec{d}\).}, we assign higher priority when retrieving \(R_1\), e.g., simulating it being more valuable than \(R_2\).

The environment hyper-parameters are defined in~\Cref{tab:grid-mo-params}.

\begin{table}[t]
\centering
\caption{Grid environment parameters for the resource-gathering environment (MO task). Grid positions are defined using 0-based indexing and \((0,0)\) position it at the top-left corner.}
\label{tab:grid-mo-params}
\begin{tabular}{c c}
\toprule
Parameter & Value \\
\midrule
Grid size & \(5\times5\) \\
Agent starting position & \((4,2)\) \\
Home position \((x_\text{home},y_\text{home})\) & \((4,2)\) \\
\(R_1\) position & \((0, 2)\) \\
\(R_2\) position & \((1, 4)\) \\
Enemies positions & \((0, 3),(1,2)\) \\
\(p_\text{slip}\) & \(0.05\) \\
\(p_\text{defeat}\) & \(0.025\) \\

\bottomrule
\end{tabular}
\end{table}

\subsection{Compute}
The illustrative environments were simulated on a laptop CPU (Intel 11th Gen i5-1135G7) with \SI{16}{\giga\byte} of RAM. Additionally, the experiments for the grid environments were deployed on a server with a dual CPU (AMD EPYC 9224 24-Core) and \SI{770}{\giga\byte} of RAM. In the latter, due to the large number of CPUs, we effectively parallelize every experiment, allocating one CPU per independent run. \cref{tab:runtime-ie,tab:runtime-ge} display the runtime obtained when running ERM-MCTS on the illustrative and grid environments, respectively.

\begin{table}[t]
\centering
\caption{ERM-MCTS runtime on illustrative environments.}
\label{tab:runtime-ie}
\begin{tabular}{c c}
\toprule
\textbf{Environment} & Runtime \\
\midrule
\textbf{MDP} & \SI{140.50 \pm 6.22}{\second} \\
\textbf{MSEE} & \SI{165.75 \pm 25.73}{\second} \\
\textbf{IL} & \SI{897.75 \pm 66.13}{\second} \\
\textbf{MO (weighted)} & \SI{136.50 \pm 1.80}{\second} \\
\textbf{MO (max)} & \SI{134.75 \pm 7.76}{\second} \\
\textbf{MO (min)} & \SI{142.25 \pm 4.76}{\second} \\
\bottomrule
\end{tabular}
\end{table}

\begin{table}[t]
\centering
\caption{ERM-MCTS runtime on grid environments.}
\label{tab:runtime-ge}
\begin{tabular}{c c}
\toprule
\textbf{Environment} & Runtime \\
\midrule
\textbf{MSEE} & \SI{2302.19 \pm 30.89}{\second} \\
\textbf{IL} & \SI{593.00 \pm 10.35}{\second} \\
\textbf{MO} & \SI{318.18 \pm 17.31}{\second}\\
\bottomrule
\end{tabular}
\end{table}

\end{document}